\documentclass{article}

\PassOptionsToPackage{numbers,compress}{natbib}
\usepackage[main, final]{neurips_2026}

\usepackage[utf8]{inputenc} 
\usepackage[T1]{fontenc}    
\usepackage{microtype}
\usepackage{graphicx}
\usepackage{subcaption}
\usepackage{wrapfig}
\usepackage{booktabs}       
\usepackage{hyperref}
\hypersetup{pdftitle={GIST: Gauge-Invariant Spectral Transformers for Scalable Graph Neural Operators}}
\usepackage{url}
\usepackage{nicefrac}

\usepackage{algorithm}
\usepackage{algpseudocode}

\usepackage{amsmath}
\usepackage{amssymb}
\usepackage{amsfonts}
\usepackage{mathtools}
\usepackage{amsthm}
\usepackage{dsfont}  

\usepackage{xcolor}
\usepackage{soul}   
\usepackage{array}  
\usepackage{enumitem}  

\theoremstyle{plain}
\newtheorem{theorem}{Theorem}[section]
\newtheorem{proposition}[theorem]{Proposition}
\newtheorem{lemma}[theorem]{Lemma}
\newtheorem{corollary}[theorem]{Corollary}
\newtheorem*{proposition*}{Proposition}
\theoremstyle{definition}

\theoremstyle{remark}
\newtheorem{remark}[theorem]{Remark}

\title{GIST: Gauge-Invariant Spectral Transformers for Scalable Graph Neural Operators}

\author{%
  Mattia Rigotti\thanks{Equal contribution.}~ \thanks{Corresponding author: \texttt{mrg@zurich.ibm.com}.} \quad
  Nicholas Thumiger\footnotemark[1] \quad
  Thomas Frick\footnotemark[1] \\
  IBM Research \\
}

\begin{document}

\maketitle

\begin{abstract}
Neural operators on irregular meshes face a fundamental tension. Spectral positional encodings, a natural choice in graph learning for capturing geometry, require cubic-complexity eigendecomposition and inadvertently break gauge invariance through numerical solver artifacts; existing efficient approximations typically sacrifice gauge symmetry by design.
Both failure modes break discretization invariance: models fail to transfer across mesh resolutions of the same domain, and similarly across different graphs of related structure in inductive settings.
We propose GIST (Gauge-Invariant Spectral Transformer), a scalable neural operator that resolves this tension by restricting attention to pairwise inner products of efficient approximate spectral embeddings.
We prove these inner products estimate an exactly gauge-invariant graph kernel at end-to-end $\mathcal{O}(N)$ complexity, and establish a formal connection between gauge invariance and discretization-invariant learning with bounded mismatch error.
To our knowledge, GIST is the first scalable graph neural operator with a provable discretization-mismatch bound.
Empirically, GIST sets state-of-the-art on the AirfRANS, ShapeNet-Car, DrivAerNet, and DrivAerNet++ mesh benchmarks (up to 750K nodes), and additionally matches strong baselines on standard graph benchmarks (e.g., 99.53\% micro-F1 on PPI).
Code available at \url{https://github.com/ibm/gist-neurips-2026}.
\end{abstract}

\section{Introduction}

Originally developed for natural language processing, Transformers \citep{vaswani2017attention} have since been adapted to data of increasing structural complexity, including sets \citep{lee2019set}, images \citep{dosovitskiy2021image}, and video \citep{bertasius2021is}. Recent work has extended them to graphs, capturing global relationships through self-attention and overcoming the long-range dependency limitations of message-passing GNNs \citep{dwivedi2022graph, zhu2023hierarchical}.

\paragraph{Two barriers to scalable graph neural operators.}
Adapting transformers to graphs introduces two distinct barriers that have not been simultaneously addressed.
First, a \emph{computational barrier}: exact spectral graph embeddings, while theoretically natural, require eigendecomposition of the graph Laplacian, scaling as $\mathcal{O}(N^3)$ in $N$, the number of nodes, which is prohibitively expensive at large scale.
Second, a \emph{gauge invariance barrier}: the inherent freedom to rotate eigenvectors, flip signs, or choose among degenerate eigenvectors is broken by exact methods through numerical solver artifacts (sign choices, eigenspace ordering, degeneracy handling, see \citealt{bronstein2017geometric}) and by approximate methods through their commitment to a specific basis.
A model trained on embeddings from one set of choices learns features that do not transfer, causing severe generalization loss in inductive learning where train and test involve different graph structures.
This becomes critical for neural operators, where different mesh resolutions of the same physical domain yield different Laplacians with distinct gauge choices, so spectral embeddings at different resolutions are incomparable, as are their attention kernels.
For models that use spectral information, gauge invariance (independence from arbitrary spectral choices) is therefore key both to robust inductive learning and to neural operators with provably bounded discretization error.

Existing approaches address one barrier at a time: (1) gauge-invariant spectral methods like SignNet \citep{lim2023signnet} and SPE \citep{huang2024stability} maintain gauge invariance but require full eigendecomposition; (2) approximate spectral methods achieve linear complexity but break gauge invariance; (3) generic linear transformers reduce attention complexity but ignore graph structure.
Recent sign- and basis-invariant methods recover gauge invariance through careful attention design but still require $\mathcal{O}(N^3)$ eigendecomposition or $\mathcal{O}(N^2)$ attention, limiting scalability.

\paragraph{Main contributions.}
We propose \emph{Gauge-Invariant Spectral Transformers (GIST)}, overcoming both barriers through a key insight: gauge invariance is preserved by restricting attention to inner products between spectral embeddings, which remain invariant under approximate spectral computations.
We show that for a broad class of efficient $\mathcal{O}(N)$ spectral embeddings $\widetilde\Phi = f(P)\cdot R$, where $f$ is any spectral filter and $R\in\mathbb{R}^{N\times r}$ is random with $\mathbb{E}[RR^\top]=\mathds{1}$, inner products are unbiased estimators of the graph kernel $K=f(P)f(P)^\top$, which depends only on $P$ and is therefore exactly gauge-invariant.
Restricting attention to these inner products and combining with linear attention yields end-to-end $\mathcal{O}(N)$ complexity, recovering gauge invariance exactly at the population level with statistical error $\mathcal{O}(r^{-1/2})$.
Our main contributions are:

\begin{enumerate}[leftmargin=*,itemsep=6pt,parsep=0pt,topsep=0pt]
\item establishing, to our knowledge, the first provable discretization-mismatch bound for a graph neural operator, by proving that gauge invariance guarantees discretization-invariant learning with bounded error;
\item proposing GIST, to our knowledge the first graph neural operator architecture that simultaneously achieves gauge invariance, end-to-end $\mathcal{O}(N)$ complexity, and global spectral expressivity, validated by SOTA results on mesh problems up to 750K nodes;
\item identifying gauge invariance breaking as the underlying failure mode shared by exact and approximate spectral methods, and demonstrating its impact on generalization in inductive settings.
\end{enumerate}

\section{Related works}
\label{sec:related}

\paragraph{Graph transformers.}
Graph transformers integrate structural information through positional encodings: shortest-path and centrality features (Graphormer \citep{ying2021do}), learnable structural encodings (LSPE \citep{dwivedi2022graph}), relative encodings (GRPE \citep{park2022grpe}), and full-Laplacian-spectrum encodings (SAN \citep{kreuzer2021rethinking}).
Spectral variants improve scalability through approximate bases (SpecFormer \citep{bo2023specformer}) or polynomial Laplacian filters (PolyFormer \citep{ma2024polyformer}) but forsake gauge invariance.
Hybrid local-global designs (GraphGPS \citep{rampavsek2022recipe}, Exphormer \citep{shirzad2023exphormer}) and tokenized approaches (TokenGT \citep{kim2022pure}, NAGphormer \citep{chen2023nagphormer}) address attention complexity, yet inherit positional-encoding overheads that grow super-linearly and none algorithmically preserve gauge invariance: GraphGPS, for instance, integrates spectral positional encodings but inherits their sign and basis ambiguities directly. Hierarchical Graph Transformer \citep{zhu2023hierarchical} reaches million-node scaling through coarsening at the cost of fine spectral information, while NodeFormer \citep{wu2022nodeformer} achieves linear-complexity attention via a kernelized Gumbel-Softmax that learns latent graph structure but operates without spectral encodings.
Consequently, these frameworks are rarely evaluated at the scale of \texttt{ogbn-arxiv}, where maintaining structural encodings is computationally prohibitive (GraphGPS, for instance, runs out of memory; see Table~\ref{tab:ppi}).

\paragraph{Sign- and basis-invariant positional encodings.}
A complementary line of work addresses gauge invariance at the architecture level: SignNet and BasisNet \citep{lim2023signnet} apply sign- and basis-equivariant networks to exact Laplacian eigenvectors, and SPE \citep{huang2024stability} extends this with stable, eigenvalue-dependent partitions of degenerate eigenspaces. These methods inherit the $\mathcal{O}(N^3)$ cost of full eigendecomposition, since they require explicit eigenvectors. GIST instead achieves invariance at the kernel level: attention depends only on gauge-invariant entries $K(i,j)=[f(P)f(P)^\top]_{ij}$, estimated by inner products of $\mathcal{O}(N)$ random spectral embeddings, eliminating eigendecomposition entirely. A non-spectral parallel paradigm is taken by PEARL \citep{kanatsoulis2025pearl}, which approximates eigenvector functions through message passing on random inputs combined with statistical pooling.

\paragraph{Scalable attention architectures.}

Recent advances tackled self-attention's quadratic scaling via cross-attention bottlenecks that map inputs to fixed-size latent representations or concepts \citep{jaegle2021perceiver,rigotti2022attentionbased}, kernel-based attention mechanisms using random feature approximations \citep{choromanski2020rethinking}, feature map decomposition methods that linearize the attention computation \citep{katharopoulos2020transformers}, and memory-efficient variants with sub-linear complexity \citep{likhosherstov2021sublinear}.
As noted by \citet{dao2024transformers}, many such linear transformer models are directly related to linear recurrent models such as state-space-models \citep{gu2021combining,gu2022efficiently, gu2023mamba, chennuruvankadara2024feature}.

\paragraph{Neural operators.}
Neural operators learn maps between continuous function spaces and ideally satisfy two properties: \emph{discretization invariance} (a single parameter set applies across mesh resolutions of the same continuum problem) and \emph{global integration} (representing nonlocal interactions via learned integral kernels) \citep{kovachki2023neural}.
Foundational families realize these differently: FNO parameterizes kernels in the spectral domain via FFT \citep{li2021fourier}, GNO via message passing on irregular meshes \citep{li2020neural}, and CNOs via continuous convolutions specified in the continuum \citep{raonic2023convolutional}.
Applying these methods to high-resolution meshes typically required downsampling, re-discretization onto regular lattices (e.g., SDF-based volumetric grids), or task-aware coarsening, since per-resolution graph sizes quickly become intractable.
Hybrid encoder-decoder designs handle complex geometries: GINO couples a graph encoder with a latent FNO on a proxy grid \citep{li2023geometryinformed}, DeepONet uses branch/trunk networks for arbitrary coordinate queries \citep{lu2021learning}, and U-NO adds a multi-resolution backbone \citep{rahman2022uno}.

\textbf{Transformers as neural operators.}
Self-attention is itself a learned, data-dependent kernel integral, naturally supporting discretization-invariant operator learning when combined with suitable positional features \citep{tsai2019transformer,yun2020are,lee2019set,jaegle2021perceiverio}. The Galerkin Transformer \citep{cao2021choose} formalized this interpretation, showing that softmax-free linear attention recovers a Petrov-Galerkin projection of an integral operator and giving attention a numerics-grounded reading as a learned discretized operator. Transolver casts PDE operator learning as attention from query coordinates to context tokens \citep{wu2024transolver}, while GNOT stabilizes operator training on irregular meshes via geometric normalization and gating \citep{hao2023gnot}. Closer to our multi-scale design, AeroGTO \citep{liu2025aerogto} and APT \citep{ju2026adaptive} fuse local message passing with global attention for operator learning. GIST shares this local-global structure, inspired by EfficientViT \citep{cai2022efficientvit}, but its global branch injects spectral graph information through gauge-invariant attention, the property behind the discretization-mismatch bound of Section~\ref{sec:theory}.

\paragraph{Positioning GIST.}
Existing graph transformers and scalable attention methods address complementary but not simultaneous challenges.
Gauge-invariant spectral methods based on sign- and basis-invariant architectures applied to exact eigenvectors (SignNet \citep{lim2023signnet}, SPE \citep{huang2024stability}) maintain theoretical expressiveness but incur significant computational costs from full eigendecomposition.
Approximate spectral methods achieve better scalability but typically do not preserve gauge invariance, which can cause generalization failures when the gauge choice differs between training and testing.
Generic linear transformers reduce attention complexity but typically ignore graph structure entirely.
GIST uniquely combines graph awareness through spectral embeddings, computational efficiency through random projections, gauge invariance through a modified attention mechanism, and linear attention for end-to-end linear scaling.
It also preserves the discretization-invariance and global integration properties needed for neural operator applications on mesh regression, unifying graph learning and continuous function approximation in a single framework.
Transformer-based operator architectures such as Transolver \citep{wu2024transolver}, GNOT \citep{hao2023gnot}, GINO \citep{li2023geometryinformed}, and the Galerkin Transformer \citep{cao2021choose} demonstrate resolution-agnosticism empirically; GIST complements this line of work by additionally providing a provable discretization-mismatch bound (Proposition~\ref{prop:neural_operator_full}).

\section{Approach}
\label{sec:approach}

\subsection{Preliminaries}

\paragraph{Self-attention and positional encoding.}
Given query, key and value representations $q_i$, $k_i$, $v_i$ of $N$ tokens with $i=1,\ldots,N$, self-attention \citep{vaswani2017attention} famously computes outputs as a weighted sum of values with attention weights determined by query-key similarities:
\begin{align}
	o_i = \sum_{j=1}^N \alpha_{ij} v_j, \quad \mbox{where } \alpha_{ij}=\mathrm{softmax}_j\left(\frac{q^\top_i k_j}{\sqrt{d}}\right).
	\label{eq:attention}
\end{align}
A key insight \citep{shaw2018selfattention} is that positional information can be injected through relative positional biases: $e_{ij} = \frac{q^\top_i k_j}{\sqrt{d}} + b_{ij}$, where $b_{ij}$ reflects distances between positions.
For graphs, this can be generalized by replacing $b_{ij}$ with distance measures reflecting the graph structure.

\paragraph{Graph Laplacian and spectral embeddings.}
The (normalized) \emph{graph Laplacian} ${{\mathcal{L}=\mathds{1}-D^{-\frac{1}{2}}A D^{-\frac{1}{2}}}}$ induces a natural metric via the \emph{resistance distance}: $\Omega(i,j)=(e_i - e_j)^\top\mathcal{L}^\dagger(e_i-e_j)$, where $e_i$ is the $i$th standard basis vector and $\mathcal{L}^\dagger$ is the Moore-Penrose pseudoinverse \citep{klein1993resistance}.
The resistance distance can be expressed via \emph{Laplacian eigenmaps}, which satisfy:
\begin{align}
	\Omega(i,j)=||\phi_i - \phi_j||^2 \quad \mbox{where} \quad (\phi_i)_k=\frac{1}{\sqrt{\lambda_k}}(u_k)_i,
	\label{eq:laplacian_eigenmaps}
\end{align}
with $\lambda_k, u_k$ being the eigenvalues and eigenvectors of $\mathcal{L}$. These eigenmaps are natural graph positional encodings \citep{dwivedi2021generalization}, but exact computation requires $\mathcal{O}(N^3)$ eigendecomposition, prohibitive for large graphs.

\paragraph{Approximate spectral embeddings and the gauge invariance problem.}
Note from \eqref{eq:laplacian_eigenmaps} that the graph metric is determined by pairwise inner products of eigenmaps, since $\Omega(i,j) = \langle\phi_i,\phi_i\rangle - 2\langle\phi_i,\phi_j\rangle + \langle\phi_j,\phi_j\rangle$.
This means we do not need the eigenvectors themselves, only an efficient way to estimate this kernel.
A broad class of randomized methods achieves this, computing embeddings as
\begin{align}
  \widetilde\Phi = f(P)\cdot R \;\in\mathbb{R}^{N\times r},
  \label{eq:spectral_embedding_class}
\end{align}
where $P$ is a graph operator (e.g., the random-walk matrix $D^{-1}A$ or the normalized Laplacian) with eigendecomposition $P=U\Lambda U^\top$, $f$ is a spectral filter ($f(P)=U\,\mathrm{diag}(f(\lambda_1),\ldots,f(\lambda_N))\,U^\top$), $R\in\mathbb{R}^{N\times r}$ is a random matrix with $\mathbb{E}[RR^\top]=\mathds{1}$, and $\tilde\phi_i\in\mathbb{R}^r$ is the $i$-th row of $\widetilde\Phi$.
Laplacian eigenmaps~\eqref{eq:laplacian_eigenmaps} are the special case $f(\lambda)=\lambda^{-1/2}$, $P=\mathcal{L}$; polynomial filters, rational filters, and random-walk diffusions also fit this template, with $f(P)\cdot R$ computable via repeated sparse matrix-vector products without forming $f(P)$ explicitly.

A key property shared by \emph{all} methods in this class is that their pairwise inner products are \emph{unbiased estimators} of the graph kernel $K(i,j):=[f(P)f(P)^\top]_{ij}$:
\begin{align}
  \mathbb{E}\!\left[\langle\tilde\phi_i,\tilde\phi_j\rangle\right] = K(i,j) = [f(P)f(P)^\top]_{ij},
  \label{eq:kernel_estimator}
\end{align}
where the expectation is over the draw of $R$ and follows from $\mathbb{E}[RR^\top]=\mathds{1}$ (proof in Lemma~\ref{lem:kernel_estimator}); when $P$ is symmetric this reduces to $K(i,j) = \sum_l f(\lambda_l)^2\,(u_l)_i(u_l)_j$.
Crucially, $K(i,j)$ is a function of $P$ alone and therefore \emph{exactly} gauge-invariant: it is unchanged by sign flips, eigenspace rotations, and numerical solver choices (Proposition~\ref{prop:exact_gauge_inv}).
The estimation error is statistical and vanishes as $\mathcal{O}(1/\sqrt{r})$ (Lemma~\ref{lem:concentration}; empirical check in Appendix~\ref{sec:kernel_convergence}).
In our implementation we use FastRP \citep{chen2019fast}, a polynomial filter $f(P)=\sum_{i=1}^{k}P^i$ computed via sparse random projections and power iterations at $\mathcal{O}(Nrk)$ cost (see Algorithm~\ref{code:fastrp}), but the guarantees above hold for any method in this class, e.g.\ graph diffusion kernels \citep{kondor2002diffusion} or Chebyshev polynomial filters \citep{hammond2011wavelets}.

Importantly, exact spectral methods can also introduce gauge dependence through eigensolver artifacts (sign choices, eigenspace ordering), making gauge invariance a concern across both exact and approximate approaches \citep{bronstein2017geometric}.

\paragraph{Motivation for gauge-invariant operations.}
While the inner products of approximate spectral embeddings estimate a gauge-invariant kernel, the embeddings themselves depend on arbitrary choices: the random matrix $R$, eigensolver conventions, and eigenvalue ordering.
Any model operating on these embeddings directly, rather than on their inner products, will learn features that do not transfer across graphs, resulting in a generalization failure.
Our approach addresses this by using approximate spectral embeddings as positional encodings, but restricting the Transformer to operate only on their gauge-invariant inner products.

\subsection{Our approach: GIST}

The key insight is that spectral embeddings of the form $\widetilde\Phi = f(P)\cdot R$ produce inner products $\langle\tilde\phi_i,\tilde\phi_j\rangle$ that are unbiased estimators of the exactly gauge-invariant kernel $K(i,j)=[f(P)f(P)^\top]_{ij}$ (\eqref{eq:kernel_estimator}).
By restricting all operations to depend on the embeddings only through these inner products, gauge invariance is preserved exactly at the population level (with only statistical noise of order $\mathcal{O}(1/\sqrt{r})$) while maintaining computational efficiency.

\begin{figure}[ht]
	\centering
	\includegraphics[width=0.9\linewidth]{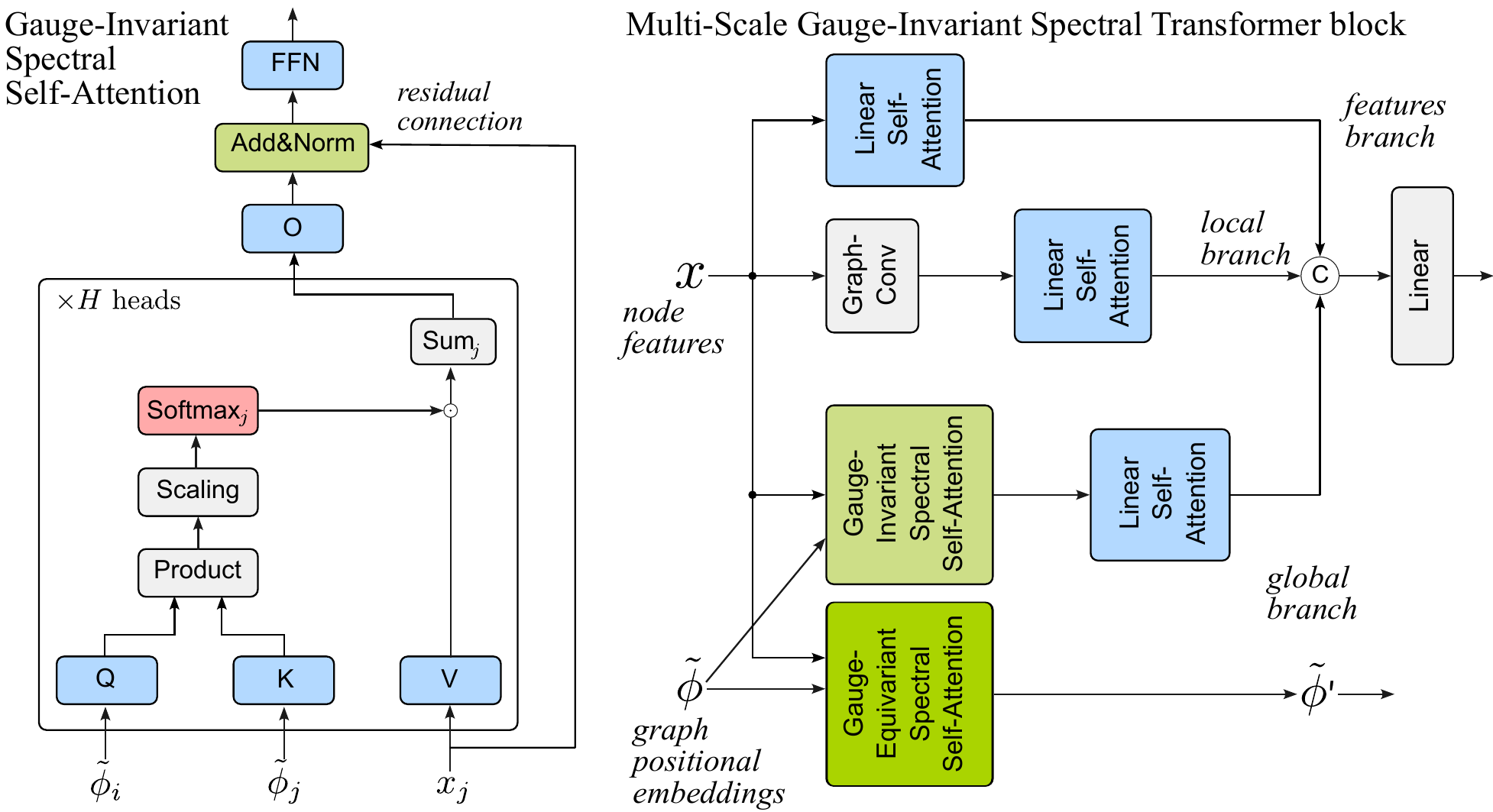}
	\caption{Gauge-Invariant Spectral Transformer.
  \textit{Left}: Gauge-Invariant Spectral Self-Attention uses spectral embeddings $\tilde\phi$ as queries and keys and node features $x$ as values;
  restricting $\tilde\phi$ to the query-key role preserves gauge invariance.
  For clarity, the panel shows the standard quadratic softmax attention implementation;
  GIST implements its linear-attention counterpart, which keeps the same queries, keys and values and has linear cost.
  \textit{Right}: the multi-scale GIST block (inspired by EfficientViT) concatenates (C) and linearly projects the node-feature outputs of three parallel branches: feature, local and global.
  Embeddings $\tilde\phi$ reach these outputs only through the attention weights of Gauge-Invariant Spectral Self-Attention;
  Gauge-Equivariant Spectral Self-Attention refines them into $\tilde\phi'$ for the next block.}
	\label{fig:diagram}
\end{figure}

\paragraph{Gauge-Invariant Spectral Self-Attention.}
We now introduce our main contribution: \emph{Gauge-Invariant Spectral Transformer (GIST)}.
The first ingredient of GIST is \emph{Gauge-Invariant Spectral Self-Attention}, which operates on approximate spectral embeddings $\tilde\phi_i \in \mathbb{R}^r$ (the $i$-th row of $\widetilde\Phi=f(P)\cdot R$) but restricts attention computations to \emph{inner products} between embeddings.
The key observation is that while the individual embeddings $\tilde\phi_i$ depend on the random matrix $R$, their inner products $\langle\tilde\phi_i,\tilde\phi_j\rangle$ are unbiased estimators of the gauge-invariant kernel $K(i,j)$ (\eqref{eq:kernel_estimator}), which depends only on $P$. Attention weights computed from these inner products are therefore gauge-invariant and transfer across graphs.

Formally, for each node $i=1,\ldots,N$, Gauge-Invariant Spectral Self-Attention modifies the standard self-attention mechanism as follows:
\begin{align*}
	q_i=\tilde\phi_i,\quad k_i=\tilde\phi_i, \quad v_i=W_v x_i,
\end{align*}
where $W_v $ is a learned linear projection of the node features $x_i$, as usual in transformers.
This ensures attention logits are based on inner products:
	$e_{ij} = q^\top_i k_j = \tilde\phi^\top_i \tilde\phi_j = K(i,j) + \mathcal{O}(1/\sqrt{r})$,
where $K(i,j)=[f(P)f(P)^\top]_{ij}$ is the exactly gauge-invariant kernel from \eqref{eq:kernel_estimator} (see Fig.~\ref{fig:diagram}, left).
By limiting the embeddings to the query-key computation and not using them as values, we ensure that downstream layers (which operate on node features) cannot access gauge-dependent information.
Algorithm~\ref{code:fastrp} in Section~\ref{app:algorithms} gives our specific implementation of the spectral embedding computation (which uses FastRP), and Algorithm~\ref{code:gauge_invariant_attention} details the attention implementation. Notice that the softmax in \eqref{eq:attention} is replaced by linear attention in our implementation (see the \textit{Linear Self-Attention} paragraph below).

\paragraph{Gauge-Equivariant Spectral Self-Attention.}
The Gauge-Invariant Spectral Self-Attention operation thus preserves gauge invariance, but at the cost of giving up a lot of the flexibility of regular self-attention.
In particular, there is no mechanism for modifying the vectors $\tilde\phi_i$ through learning while preserving the property that attention logits track the gauge-invariant kernel $K$. Applying a learnable linear map $W$ to $\tilde\phi_i$, for instance, would replace $K$ with the parameter-dependent rescaling $\bigl(\mathrm{Tr}(W^\top W)/r\bigr) \cdot K$ in expectation.
A feature-dependent scalar rescaling $s(x_i)\tilde\phi_i$, in contrast, preserves the kernel structure as a feature-driven modulation: $(s(x_i)\tilde\phi_i)^\top (s(x_j)\tilde\phi_j)=s(x_i)s(x_j)(\tilde\phi_i^\top \tilde\phi_j) = s(x_i)s(x_j)\cdot K(i,j)+ \mathcal{O}(1/\sqrt{r})$, an \emph{equivariant} operation in node features.

Remarkably, such a gauge-equivariant operation can also be straight-forwardly implemented via a modification of self-attention by modifying \eqref{eq:attention} as follows:
\begin{align*}
	q_i=W_q x_i,\quad k_i=W_k x_i, \quad v_i=\tilde\phi_i,
\end{align*}
and the output in \eqref{eq:attention} such that it is constrained to operate on $\tilde\phi_i$, i.e.\
$\tilde\phi_i^{l+1} = \sum_{j=1}^N \alpha_{ij} v_j$, where $\tilde\phi_i^{l+1}$ indicates the graph positional encoding that will be used in the next layer $l+1$.
Algorithm~\ref{code:gauge_equivariant_attention} in Section~\ref{app:algorithms} details how the implementation of Gauge-Equivariant Spectral Self-Attention relates to regular Self-Attention.

Here $W_q, W_k$ act on node features only, and no learnable map acts on the embedding coordinates: these depend on the random draw of $R$, which differs from graph to graph, so weights fitted to them would be specific to one draw (in expectation over $R$, a linear map only rescales $K$, as noted above).
The block thus generalizes the per-node rescaling $s(x_i)\tilde\phi_i$ to feature-dependent mixing across nodes, with coefficients $\alpha_{ij}$ computed from node features alone: like the rescaling, it combines embeddings of different nodes and never mixes the coordinates of an embedding.
Combining the two blocks recovers learnable query-key projections within a gauge-respecting representation: the gauge-invariant block exposes $K(i,j)$ as attention logits, while the gauge-equivariant block reinstates $W_q, W_k$ on node features with $\tilde\phi$ as values.

\paragraph{Linear self-attention, and multi-scale architecture.}
Gauge-Invariant Spectral Self-Attention ensures that we can compute reliable graph positional encoding with linear time complexity in the number of nodes in the graph $N$.
In order to maintain that linear scaling end-to-end, all GIST attention blocks replace softmax attention (introduced above and depicted in Fig.~\ref{fig:diagram}, left, for clarity) with the \emph{linear transformer} by \citet{katharopoulos2020transformers}, using a feature map $\varphi$ so that $\alpha_{ij}\propto\langle\varphi(q_i),\varphi(k_j)\rangle$.

Crucially, as feature map we use $\varphi(x)=\mathrm{ReLU}(x)$, which is known to correspond to a well-defined kernel (the first-order arc-cosine kernel \citep{cho2009kernel}).
In practice, we use $\varphi(x)=\mathrm{ReLU}(x-\theta)$, with a trainable scalar $\theta$.
Gauge invariance of the attention weights is established at the sample level: for any fixed realization of $R$, $\tilde\phi_i=[f(P)R]_i$ is a deterministic function of the gauge-independent operator $P$ (no eigendecomposition is performed), so $\langle\varphi(\tilde\phi_i),\varphi(\tilde\phi_j)\rangle$ is exactly invariant under sign flips, eigenspace rotations, and solver choices. In expectation, the inner products estimate the gauge-invariant kernel $K(i,j)=[f(P)f(P)^\top]_{ij}$ (\eqref{eq:kernel_estimator}), connecting the population-level attention to $K$.
For further considerations on the choice of the feature map $\varphi(\cdot)$ see Appendix \ref{app:algorithms}.

To mitigate linear attention's drawbacks (notably lack of sharp attention scores compared to softmax attention), we design a parallel architecture inspired by EfficientViT \citep{cai2022efficientvit}, a multi-scale linear attention architecture.
Our \emph{Multi-Scale Gauge-Invariant Spectral Transformer Block} comprises three parallel branches: a feature branch (linear transformer on node features $x$), a local branch (graph convolution on $x$ followed by linear transformer), and a global branch (Gauge-Invariant Spectral Self-Attention on $x$ and $\tilde\phi$, followed by linear transformer). In parallel, Gauge-Equivariant Spectral Self-Attention on $x$ and $\tilde\phi$ produces the refined embeddings used by the next block.
The local branch emphasizes local information that linear attention would otherwise diffuse, while the global branch integrates global graph information.
See Figure~\ref{fig:diagram} (right panel) for the architecture.
Ablation studies in Appendix~\ref{sec:ablation} confirm that each branch meaningfully contributes to performance.
This parallel design also covers the alternative construction of concatenating node features and embeddings to form queries and keys.
Without learnable projections, the feature map acts blockwise on the concatenation, $\varphi([x_i;\tilde\phi_i])=[\varphi(x_i);\varphi(\tilde\phi_i)]$, so the attention score splits exactly into $\langle\varphi(x_i),\varphi(x_j)\rangle+\langle\varphi(\tilde\phi_i),\varphi(\tilde\phi_j)\rangle$, and the linear-attention output is a per-node convex combination, weighted by the two normalizers, of a feature-only attention and Gauge-Invariant Spectral Self-Attention sharing one value projection.
Our branches compute the two terms separately, adding learnable $W_q, W_k$ to the feature term and separate output projections; a learnable projection of the concatenated vector would instead act on the embedding coordinates, which we avoid as discussed above.

\paragraph{Complexity scaling analysis.}
GIST computational complexity is dominated by two components.
First, spectral embedding computation scales as $\mathcal{O}(N \cdot r \cdot k)$, where $N$ is the number of nodes, $r$ is the embedding dimension, and $k$ is the number of power iterations (e.g., $k$ polynomial filter steps for FastRP).
Second, linear transformer blocks combining Gauge-Invariant and Gauge-Equivariant Spectral Self-Attention scale as $\mathcal{O}(N \cdot d^2 + N \cdot r \cdot d)$ on $d$-dimensional node features and $r$-dimensional spectral embeddings, where the $N \cdot d^2$ term comes from the feature projections and feature-only linear attention, and the $N \cdot r \cdot d$ term from the attention steps that couple $r$-dim spectral embeddings with $d$-dim node features.
Overall, this gives end-to-end $\mathcal{O}(N \cdot d^2 + N \cdot r \cdot d + N \cdot r \cdot k)$ scaling, i.e.\ linear in $N$.
This contrasts with $\mathcal{O}(N^3)$ for exact eigendecomposition and $\mathcal{O}(N^2 d)$ for standard quadratic attention.
This linear scaling is empirically validated in Appendix~\ref{sec:scalability}: Figure~\ref{fig:gist_scalability} demonstrates that both VRAM consumption and forward pass time grow linearly with graph size up to 500K nodes on DrivAerNet samples, confirming our theoretical analysis and enabling the large-scale neural operator experiments in Section~\ref{sec:results}.

\section{Results}
\label{sec:results}

This section validates GIST both theoretically and empirically: a provable discretization-mismatch bound (\S\ref{sec:theory}), a controlled cross-resolution mechanism check (\S\ref{sec:disc_verification}), node classification on five standard graph benchmarks, and surface-pressure regression on four neural operator benchmarks (\S\ref{sec:neural_operators}).

\subsection{Theoretical guarantees: GIST as discretization-invariant neural operator}
\label{sec:theory}

We now formalize how gauge invariance enables discretization-invariant learning with bounded error.
Physical problems (computational fluid dynamics, structural mechanics, shape analysis) are discretized into computational meshes (i.e.\ graphs), with different resolutions producing different graph Laplacians and spectral decompositions involving arbitrary gauge choices (sign flips, eigenspace rotations, solver artifacts).
Without gauge invariance, parameters fail to transfer across discretizations; gauge invariance ensures attention weights converge to a well-defined continuum kernel with bounded discretization mismatch error.

\begin{proposition}[Informal formulation of Corollary~\ref{cor:output_mismatch}]
    \label{prop:neural_operator_informal}
    The output of a gauge-invariant spectral self-attention block of GIST, applied to two nested random samplings of the same $m$-dimensional manifold, converges across discretizations with error $\mathcal{O}(n^{-1/(m+4)}) + \mathcal{O}(r^{-1/2})$, where $n$ is the coarser resolution and $r$ is the spectral embedding dimension.
    This ensures learned parameters transfer across mesh resolutions with bounded error that vanishes as resolution and embedding dimension increase.
\end{proposition}

\textit{Proof sketch.}
Cross-discretization output mismatch decomposes additively into a discretization error and a statistical error. Graph Laplacians of $n$-point samples of an $m$-dimensional manifold converge spectrally to the Laplace-Beltrami operator at rate $\mathcal{O}(n^{-1/(m+4)})$ \citep{belkin2008towards, calder2022improved}, so any kernel $K=f(P)f(P)^\top$ from a smooth spectral filter converges across discretizations at the same rate (Corollary~\ref{cor:kernel_convergence}). Inner products of $\widetilde\Phi = f(P)\cdot R$ unbiasedly estimate $K$ with statistical error $\mathcal{O}(r^{-1/2})$ (Lemma~\ref{lem:concentration}), and $K$ depends only on $P$, hence is exactly gauge-invariant (Proposition~\ref{prop:exact_gauge_inv}). Since gauge-invariant attention depends on the embeddings only through these inner products, the triangle inequality yields an output-mismatch bound $\mathcal{O}(n^{-1/(m+4)}) + \mathcal{O}(r^{-1/2})$.
Full statements and proofs in Appendix~\ref{app:proofs}.

This result distinguishes GIST from prior spectral methods: while approximate methods achieve computational efficiency, they lack gauge invariance and thus cannot provide discretization error bounds; while exact methods could maintain invariance, their $\mathcal{O}(N^3)$ cost prevents scaling to high-resolution meshes where such bounds are most valuable.
GIST uniquely combines gauge invariance with end-to-end linear complexity, achieving both theoretical guarantees and practical scalability for neural operator applications.

\subsection{Discretization invariance: controlled mechanism check}
\label{sec:disc_verification}

To isolate the mechanism predicted by Proposition~\ref{prop:neural_operator_informal}, we construct a minimal setup in which gauge invariance is the only architectural difference between two matched models, leaving the broader empirical assessment of the full architecture to the operator-learning benchmarks of Section~\ref{sec:neural_operators}.

\begin{table}[h]
	\small
	\setlength{\tabcolsep}{4pt}
	\centering
	\caption{Controlled discretization-invariance check on DrivAerNet surface pressure. One model per architecture is trained at 50\% decimation and evaluated without retraining on the same cars at four resolutions ($R^2$, mean $\pm$ std across 600 cars). Both architectures match at the training resolution; away from it, the non-gauge-invariant baseline loses up to 47\% of its $R^2$, the gauge-invariant one at most 19\%.}
	\label{tab:disc_invariance}
	\smallskip
	\begin{tabular}{lcc}
		\toprule
		\textbf{\begin{tabular}[t]{@{}l@{}}\\Evaluation resolution\end{tabular}} &
		\textbf{\begin{tabular}[t]{@{}c@{}}GIST (gauge-invariant)\\$R^2$ $\uparrow$\end{tabular}} &
		\textbf{\begin{tabular}[t]{@{}c@{}}Non-gauge-invariant\\$R^2$ $\uparrow$\end{tabular}} \\
		\midrule
		25\% (${\sim}$125K) & \textbf{0.796}{\tiny $\pm$0.028} & 0.541{\tiny $\pm$0.041} \\
		50\% (${\sim}$250K), training res. & 0.987{\tiny $\pm$0.012} & 0.989{\tiny $\pm$0.018} \\
		75\% (${\sim}$375K) & \textbf{0.861}{\tiny $\pm$0.018} & 0.702{\tiny $\pm$0.035} \\
		100\% (${\sim}$500K) & \textbf{0.836}{\tiny $\pm$0.021} & 0.524{\tiny $\pm$0.033} \\
		\bottomrule
	\end{tabular}
\end{table}

\paragraph{Experimental setup.}
We train two architectures with matched parameters on 600 cars of DrivAerNet \citep{elrefaie2024drivaernet} at 50\% decimation (${\sim}$250K vertices) and evaluate them without retraining on the same cars at 25, 50, 75 and 100\% of the vertices, predicting surface pressure.
The \emph{gauge-invariant} architecture is a GIST block followed by a linear transformer block.
The \emph{non-gauge-invariant} baseline uses a linear transformer block with spectral embeddings summed to node features (treating them as standard positional encodings), which breaks gauge invariance by using the projected embeddings directly as features rather than restricting attention to their inner products.

\paragraph{Results.}
Table~\ref{tab:disc_invariance} shows both architectures achieve the same $R^2$ at the training resolution (0.987 vs 0.989), confirming matched model capacity.
Away from it, the gauge-invariant architecture degrades gradually (at most a 19\% relative drop, 0.796 $R^2$ at 25\%), while the non-gauge-invariant baseline fails catastrophically (up to a 47\% relative drop, 0.524 $R^2$ at full resolution).
Within the scope of this controlled comparison, gauge invariance is the sole architectural factor that distinguishes successful cross-resolution transfer from catastrophic failure, consistent with the mechanism predicted by Proposition~\ref{prop:neural_operator_informal}.
Exact gauge-invariant methods (SignNet \citep{lim2023signnet}, SPE \citep{huang2024stability}) are gauge-invariant by construction and inherit the same cross-resolution transfer property. Their $\mathcal{O}(N^3)$ eigendecomposition, however, restricts them to the small-graph regime and does not apply at the mesh resolutions of Section~\ref{sec:neural_operators}, where the discretization-mismatch bound becomes consequential.
The same mechanism explains why gauge invariance helps little on the node-classification benchmarks (Section~\ref{sec:node_classification}): on the single-graph benchmarks, train and test share one graph and hence one gauge, and on PPI the strongest baselines are message-passing GNNs, gauge-invariant by construction.

\subsection{Node classification tasks}
\label{sec:node_classification}

We evaluate GIST on six standard node-classification benchmarks: the Planetoid citation graphs Cora, CiteSeer and PubMed \citep{sen2008collective,yang2016revisiting,kipf2017semisupervised}, the Amazon Photo co-purchase network and the ogbn-arxiv citation graph \citep{fey2019fast,fey2025pyg}, and the PPI protein interaction benchmark.
The first five are single-graph benchmarks where train, validation, and test are disjoint node subsets of one graph; PPI consists of 24 disjoint tissue-specific protein–protein interaction graphs and is the multi-graph case in this suite, where train and test involve different graph structures and gauge invariance becomes most consequential.
We follow the official splits in each case, reporting micro-averaged F1 on the multi-label PPI benchmark and accuracy elsewhere.
Spectral embedding hyperparameters ($k$ and $r$) are selected by HPO; Appendix~\ref{sec:hp_robustness} confirms robust performance across a wide range, with accuracy saturating at relatively low $r$, consistent with the $\mathcal{O}(r^{-1/2})$ concentration bound (Lemma~\ref{lem:concentration}).

\begin{table}[h]
	\small
	\setlength{\tabcolsep}{3pt}
	\centering
	\caption{Node classification on standard graph benchmarks.
    GIST is competitive across all six benchmarks at end-to-end $\mathcal{O}(N)$ complexity, with the best mean accuracy on CiteSeer, PubMed and Photo, and matching SOTA on PPI.
    Cora, CiteSeer, PubMed, Photo, and ArXiv report test accuracy;
    PPI reports micro-F1;
    all metrics are percentages and higher-is-better, reported as mean $\pm$ std across 6 random seeds for GIST.
    Sources are cited per model for baselines.
    $^*$source reports 95\% CI rescaled to std.
    $^\dagger$source does not specify the uncertainty type.
    `--' indicates results not reported by the source.
    Bold denotes the best mean in each column.
    }
	\label{tab:ppi}
	\smallskip
	\begin{tabular*}{\textwidth}{@{\extracolsep{\fill}} l l l l l l l @{}}
		\toprule
		\textbf{Model} &
		\textbf{\begin{tabular}[t]{@{}c@{}}Cora\\{\scriptsize (accuracy $\uparrow$)}\end{tabular}} &
		\textbf{\begin{tabular}[t]{@{}c@{}}CiteSeer\\{\scriptsize (accuracy $\uparrow$)}\end{tabular}} &
		\textbf{\begin{tabular}[t]{@{}c@{}}PubMed\\{\scriptsize (accuracy $\uparrow$)}\end{tabular}} &
		\textbf{\begin{tabular}[t]{@{}c@{}}Photo\\{\scriptsize (accuracy $\uparrow$)}\end{tabular}} &
		\textbf{\begin{tabular}[t]{@{}c@{}}ArXiv\\{\scriptsize (accuracy $\uparrow$)}\end{tabular}} &
		\textbf{\begin{tabular}[t]{@{}c@{}}PPI\\{\scriptsize (micro-F1 $\uparrow$)}\end{tabular}} \\
		\midrule
		GCN \citep{luo2024classic,zeng2020graphsaint,brody2022how,bo2023specformer}       & 85.10{\tiny $\pm$0.67} & 73.14{\tiny $\pm$0.67} & 81.12{\tiny $\pm$0.52} & 88.26{\tiny $\pm$0.37$^*$} & 71.74{\tiny $\pm$0.29} & 51.50{\tiny $\pm$0.60$^\dagger$} \\
		GraphSAGE \citep{luo2024classic,velickovic2018graph,brody2022how}        & 83.88{\tiny $\pm$0.65} & 72.26{\tiny $\pm$0.55} & 79.72{\tiny $\pm$0.50} & \phantom{00.}--\phantom{{\tiny $\pm$0.00}} & 71.49{\tiny $\pm$0.27} & 76.80\phantom{{\tiny $\pm$0.00}} \\
		GAT \citep{luo2024classic,velickovic2018graph,brody2022how,bo2023specformer}   & 84.46{\tiny $\pm$0.55} & 72.22{\tiny $\pm$0.84} & 80.28{\tiny $\pm$0.64} & 90.94{\tiny $\pm$0.34$^*$} & 71.59{\tiny $\pm$0.38} & 97.30{\tiny $\pm$0.02} \\
		GATv2 \citep{choi2022personalized,brody2022how}            & 82.90\phantom{{\tiny $\pm$0.00}} & 71.60\phantom{{\tiny $\pm$0.00}} & 78.70\phantom{{\tiny $\pm$0.00}} & \phantom{00.}--\phantom{{\tiny $\pm$0.00}} & 71.87{\tiny $\pm$0.25} & 96.30\phantom{{\tiny $\pm$0.00}} \\
		GaAN \citep{zhang2018gaan}  & \phantom{00.}--\phantom{{\tiny $\pm$0.00}} & \phantom{00.}--\phantom{{\tiny $\pm$0.00}} & \phantom{00.}--\phantom{{\tiny $\pm$0.00}} & \phantom{00.}--\phantom{{\tiny $\pm$0.00}} & \phantom{00.}--\phantom{{\tiny $\pm$0.00}} & 98.71\phantom{{\tiny $\pm$0.00}} \\
		Cluster-GCN \citep{chiang2019clustergcn}      & \phantom{00.}--\phantom{{\tiny $\pm$0.00}} & \phantom{00.}--\phantom{{\tiny $\pm$0.00}} & \phantom{00.}--\phantom{{\tiny $\pm$0.00}} & \phantom{00.}--\phantom{{\tiny $\pm$0.00}} & \phantom{00.}--\phantom{{\tiny $\pm$0.00}} & 99.36\phantom{{\tiny $\pm$0.00}} \\
		GraphSAINT \citep{zeng2020graphsaint}       & \phantom{00.}--\phantom{{\tiny $\pm$0.00}} & \phantom{00.}--\phantom{{\tiny $\pm$0.00}} & \phantom{00.}--\phantom{{\tiny $\pm$0.00}} & \phantom{00.}--\phantom{{\tiny $\pm$0.00}} & \phantom{00.}--\phantom{{\tiny $\pm$0.00}} & 99.50{\tiny $\pm$0.00$^\dagger$} \\
		GCNII \citep{chen2020simple,bo2023specformer}  & \textbf{85.50}{\tiny $\pm$0.50} & 73.40{\tiny $\pm$0.60} & 80.20{\tiny $\pm$0.40} & 89.94{\tiny $\pm$0.16$^*$} & 72.04{\tiny $\pm$0.10$^*$} & \textbf{99.53}{\tiny $\pm$0.01} \\
		GCNIII \citep{chen2025wide}           & \phantom{00.}--\phantom{{\tiny $\pm$0.00}} & \phantom{00.}--\phantom{{\tiny $\pm$0.00}} & \phantom{00.}--\phantom{{\tiny $\pm$0.00}} & \phantom{00.}--\phantom{{\tiny $\pm$0.00}} & \phantom{00.}--\phantom{{\tiny $\pm$0.00}} & 99.50{\tiny $\pm$0.03} \\
		\midrule
		GraphGPS \citep{rampavsek2022recipe}       & \phantom{00.}--\phantom{{\tiny $\pm$0.00}} & \phantom{00.}--\phantom{{\tiny $\pm$0.00}} & \phantom{00.}--\phantom{{\tiny $\pm$0.00}} & \phantom{00.}--\phantom{{\tiny $\pm$0.00}} & OOM\phantom{{\tiny $\pm$0.00}} & \phantom{00.}--\phantom{{\tiny $\pm$0.00}} \\
		SGFormer \citep{luo2024classic,wu2023sgformer}       & 84.82{\tiny $\pm$0.85} & 72.72{\tiny $\pm$1.15} & 80.60{\tiny $\pm$0.49} & \phantom{00.}--\phantom{{\tiny $\pm$0.00}} & \textbf{72.63}{\tiny $\pm$0.13} & \phantom{00.}--\phantom{{\tiny $\pm$0.00}} \\
		SpecFormer \citep{bo2023specformer}       & \phantom{00.}--\phantom{{\tiny $\pm$0.00}} & \phantom{00.}--\phantom{{\tiny $\pm$0.00}} & \phantom{00.}--\phantom{{\tiny $\pm$0.00}} & 95.48{\tiny $\pm$0.16$^*$} & 72.37{\tiny $\pm$0.09$^*$} & \phantom{00.}--\phantom{{\tiny $\pm$0.00}} \\
		PolyFormer \citep{ma2024polyformer}       & \phantom{00.}--\phantom{{\tiny $\pm$0.00}} & \phantom{00.}--\phantom{{\tiny $\pm$0.00}} & \phantom{00.}--\phantom{{\tiny $\pm$0.00}} & \phantom{00.}--\phantom{{\tiny $\pm$0.00}} & 72.42{\tiny $\pm$0.19} & \phantom{00.}--\phantom{{\tiny $\pm$0.00}} \\
		Exphormer (LapPE) \citep{shirzad2023exphormer} & \phantom{00.}--\phantom{{\tiny $\pm$0.00}} & \phantom{00.}--\phantom{{\tiny $\pm$0.00}} & \phantom{00.}--\phantom{{\tiny $\pm$0.00}} & 95.27{\tiny $\pm$0.42} & 72.44{\tiny $\pm$0.28$^\dagger$} & \phantom{00.}--\phantom{{\tiny $\pm$0.00}} \\
		\midrule
		\textbf{GIST (ours)} & 84.25{\tiny $\pm$0.65} & \textbf{73.85}{\tiny $\pm$0.63} & \textbf{82.60}{\tiny $\pm$0.37} & \textbf{95.65}{\tiny $\pm$0.46} & 72.31{\tiny $\pm$0.29} & \textbf{99.53}{\tiny $\pm$0.01} \\
		\bottomrule
	\end{tabular*}
\end{table}

GIST is competitive across all six benchmarks (Table~\ref{tab:ppi}), with four notable highlights: best mean accuracy on CiteSeer ($73.85\%\pm0.63$), PubMed ($82.60\%\pm0.37$), Photo ($95.65\%\pm0.46$, on par with SpecFormer's $95.48\%\pm0.16$), and matching SOTA on PPI ($99.53\%\pm0.01$ micro-F1, matching GCNII).
On Cora and ogbn-arxiv, GIST stays within $1.3$ and $0.4$ points of the best reported method, at end-to-end $\mathcal{O}(N)$ complexity (GraphGPS goes OOM on ogbn-arxiv).

\subsection{Neural operators}
\label{sec:neural_operators}

We validate GIST's discretization-invariant properties (Section~\ref{sec:theory}) on four surface pressure prediction benchmarks of increasing scale, all involving node-level regression of aerodynamic fields on unseen geometries.

AirfRANS \citep{bonnet2022airfrans} is a high-fidelity computational fluid dynamics dataset of 2D airfoils under Reynolds-Averaged Navier-Stokes (RANS) simulation, comprising 1000 airfoil designs with approximately 200K mesh nodes per sample across Reynolds numbers 2--6M and angles of attack from $-5^\circ$ to $15^\circ$.
ShapeNet-Car \citep{chang2015shapenet} is a collection of 3D car geometries used for aerodynamic prediction, containing approximately 3700 mesh points per sample.
DrivAerNet and DrivAerNet++ \citep{elrefaie2024drivaernet} are distinct high-fidelity CFD datasets of parametric car geometries: DrivAerNet comprises 4{,}000 designs with approximately 500K surface vertices per car, while DrivAerNet++ extends this to over 10{,}000 designs with up to ${\sim}$750K vertices per car and greater geometric diversity across multiple vehicle classes (SUVs, sedans, hatchbacks).
All datasets model cars or airfoils as graphs with nodes representing surface mesh vertices and edges following mesh connectivity.

\begin{wraptable}{r}{0.45\textwidth}
\vspace{-1em}
\small
\centering
    \caption{Surface pressure prediction error on AirfRANS and ShapeNet-Car.
      GIST outperforms Transolver on both, with an 8$\times$ reduction in MSE on AirfRANS.
      Baseline is Transolver \citep{wu2024transolver}.
      Reported metrics are MSE on AirfRANS and relative $\ell_2$ on ShapeNet-Car;
      GIST mean $\pm$ std over 3 seeds.
      }
    \label{tab:airfrans_shapenet}
    {
    \setlength{\tabcolsep}{3pt}\begin{tabular*}{\linewidth}{@{\extracolsep{\fill}} l cc @{}}
        \toprule
        \textbf{Model} & \textbf{AirfRANS $\downarrow$} & \shortstack{\textbf{ShapeNet-Car} $\downarrow$} \\
        \midrule
        Transolver & 0.0142\phantom{\tiny $\pm$0.0000} & 7.45\%\phantom{\tiny$\pm$0.00} \\
        \textbf{GIST (ours)} & \textbf{0.0017}{\tiny $\pm$0.0002} & \textbf{7.02\%}{\tiny$\pm$0.05} \\
        \bottomrule
    \end{tabular*}}
\vspace{-1em}
\end{wraptable}
On the intermediate-scale benchmarks (AirfRANS and ShapeNet-Car, Table~\ref{tab:airfrans_shapenet}), GIST outperforms Transolver \citep{wu2024transolver}, achieving 0.0017 MSE on AirfRANS (vs.\ 0.0142) and 7.02\% relative $\ell_2$ on ShapeNet-Car (vs.\ 7.45\%).
Both results follow the same train/test splits and surface-pressure metrics (MSE over surface nodes on AirfRANS, relative $\ell_2$ per car on ShapeNet-Car) as reported by Transolver, using the standard AirfRANS protocol of \citet{bonnet2022airfrans}.
On the large-scale benchmarks (DrivAerNet and DrivAerNet++, Table~\ref{tab:drivaernet_comparison}), GIST achieves state-of-the-art on both: 20.03\% relative $\ell_2$ error on DrivAerNet (vs.\ 20.35\% for the previous best) and 18.60\% on DrivAerNet++ (vs.\ 20.05\%), outperforming RegDGCNN \citep{elrefaie2024drivaernet}, Transolver \citep{wu2024transolver}, FigConvNet \citep{choy2025factorized}, and TripNet \citep{chen2025tripnet}.
GIST's linear scaling enables direct processing of these meshes (up to ${\sim}$750K vertices on DrivAerNet++) without downsampling, while prior methods require lossy projection to regular grids or lower-dimensional latent spaces.
GIST's spectral attention provides global receptive fields in a single layer, which aids performance on all these tasks.

\begin{table}[h]
\normalsize
\centering
    \caption{Surface pressure prediction error on DrivAerNet and DrivAerNet++. GIST sets state-of-the-art on both benchmarks, processing meshes directly at full resolution.}
    \label{tab:drivaernet_comparison}
    \smallskip
    \begin{tabular*}{0.8\textwidth}{@{\extracolsep{\fill}} l cc cc @{}}
        \toprule
        & \multicolumn{2}{c}{\textbf{DrivAerNet}} & \multicolumn{2}{c}{\textbf{DrivAerNet++}} \\
        \cmidrule(r){2-3} \cmidrule(l){4-5}
        \textbf{Model} & \textbf{MSE $\downarrow$} & \textbf{Rel $\ell_2$ $\downarrow$} & \textbf{MSE $\downarrow$} & \textbf{Rel $\ell_2$$\downarrow$} \\
        \midrule
        RegDGCNN        & $9.01\times 10^{-2}$ & 28.49\% & $8.29\times 10^{-2}$ & 27.72\% \\
        Transolver      & $5.37\times 10^{-2}$ & 22.52\% & $7.15\times 10^{-2}$ & 23.87\% \\
        FigConvNet      & $4.38\times 10^{-2}$ & 20.98\% & $4.99\times 10^{-2}$ & 20.86\% \\
        TripNet         & $4.23\times 10^{-2}$ & 20.35\% & $4.55\times 10^{-2}$ & 20.05\% \\
        \midrule
        \textbf{GIST (ours)} & $\mathbf{4.00 \times 10^{-2}}$ & \textbf{20.03\%} & $\mathbf{3.63 \times 10^{-2}}$ & \textbf{18.60\%} \\
        \bottomrule
    \end{tabular*}
\end{table}

\section{Conclusions}
\label{sec:conclusion}

We presented GIST, a gauge-invariant spectral transformer that sets state-of-the-art on four mesh-regression benchmarks at up to 750K nodes, with contributions spanning three levels.
First, we proved formally that gauge invariance guarantees discretization-invariant neural operators, establishing, to our knowledge, the first provable discretization-mismatch bound $\mathcal{O}(n^{-1/(m+4)}) + \mathcal{O}(r^{-1/2})$ for a graph neural operator.
Second, we proposed an architecture that achieves gauge invariance, end-to-end $\mathcal{O}(N)$ complexity, and global spectral expressivity simultaneously, three properties that no prior method combines.
Third, we identified gauge invariance breaking as the underlying failure mode shared by both exact and approximate spectral methods, causing severe generalization loss in inductive settings.
By restricting attention to inner products of spectral embeddings, which are unbiased estimators of an exactly gauge-invariant kernel, GIST recovers the symmetry algorithmically while maintaining linear scaling.
Empirically, GIST sets a new state-of-the-art on mesh regression across AirfRANS, ShapeNet-Car, DrivAerNet, and DrivAerNet++ (up to 500K and 750K nodes for the latter two), processing these meshes directly at full resolution where methods requiring quadratic attention or cubic eigendecomposition cannot scale; on standard graph benchmarks (Cora, CiteSeer, PubMed, Photo, ogbn-arxiv, PPI), GIST achieves competitive results.

\paragraph{Broader impact, limitations, and future work.}
GIST is a methodological contribution to graph neural operators, with applications in mesh-based scientific computing such as CFD-based aerodynamic and structural simulation; we do not foresee direct negative societal impacts beyond those generic to scientific simulation tooling. On the theoretical side, the discretization-mismatch bound rests on the assumptions of \citet{calder2022improved}, namely i.i.d.\ uniform sampling on a bounded-curvature manifold, conditions that CFD meshes, which are refined non-uniformly (more finely where geometry or flow fields vary most), satisfy only approximately. All our mesh experiments use the benchmarks' non-uniform meshes, so the reported results already reflect the effect of non-uniform sampling; extending the bound to non-uniform sampling is a natural next step. Empirically, our scope spans homophilic graphs and physical meshes, leaving heterophilic settings for future investigation; the benefit of gauge invariance is also limited on single-graph benchmarks, where train and test share one gauge (Section~\ref{sec:disc_verification}). We also precompute the spectral embedding under the assumption of a static input graph, and adapting GIST to dynamic graphs whose topology evolves over time, by recomputing or updating the embedding online, is left to future work.

\begin{ack}
This work was supported by the European Union's Horizon Europe research and innovation programme through the NATURE-DEMO project under grant agreement No 101157448, and by the Swiss State Secretariat for Education, Research and Innovation (SERI) under contract No 24.00279. The authors declare no competing interests.
\end{ack}

\newpage

\bibliographystyle{unsrtnat}
\bibliography{bibliography}

\newpage
\appendix
\section{Proofs}
\label{app:proofs}

\subsection{Proof of Proposition~\ref{prop:neural_operator_full}: Gauge-Invariant Spectral Self-Attention}

We prove that the Gauge-Invariant Spectral Self-Attention mechanism is discretization-invariant with quantifiable discretization mismatch error through three stages of analysis. The proof establishes that the attention kernel $K(i,j)=[f(P)f(P)^\top]_{ij}$ estimated by the spectral embeddings converges to a well-defined continuum kernel as mesh resolution increases, allowing us to bound the discretization mismatch error between any two discretizations of the same manifold. We carry out the analysis on the rescaled Laplacian $\mathcal{L}_n$, for which sharp spectral-convergence rates are available \citep{calder2022improved}; the same conclusions extend to rescaled variants of $\mathcal{L}_n$, including the random-walk operator $P_n = D^{-1}A$ used in our FastRP implementation (Remark~\ref{rem:rescaled_variants} below).

For completeness, we restate the full proposition:

\begin{proposition}[Full Statement]
\label{prop:neural_operator_full}
Gauge-Invariant Spectral Self-Attention is a discretization-invariant Neural Operator with bounded discretization mismatch error.
Let $M$ be a compact $m$-dimensional Riemannian manifold and $\mathcal{G}_n$ a sequence of graphs obtained by sampling $n$ points $\{x_i\}_{i=1}^n\subset M$ as nodes, with $n\to\infty$.
Let $f$ be a spectral filter applied to a graph operator $P_n$ of $\mathcal{G}_n$, and let $\phi_i^n = f(P_n)_{i\cdot}^\top$ be the spectral embedding of node $i$ (with $\ker(P_n)$ projected out when $f$ has a pole at $0$).
Then:
\begin{enumerate}
    \item[(i)] \textbf{(Convergence to a continuum kernel.)} When $f(\lambda)=\lambda^{-1/2}$ and $P_n=\mathcal{L}_n$, the inner products $\langle \phi_i^n, \phi_j^n \rangle$ converge to the Green's function $G_M(x_i, x_j)$ of the Laplace-Beltrami operator at rate $\mathcal{O}(n^{-1/(m+4)})$ (Proposition~\ref{prop:stage1_spectral_convergence}). For spectral filters $f$ for which $g(\lambda):=f(\lambda)^2$ is Lipschitz away from $0$ and $\sum_l f(\mu_l)^2 < \infty$, the kernel $K_n(i,j)=\langle\phi_i^n,\phi_j^n\rangle = [f(P_n)f(P_n)^\top]_{ij}$ converges to a continuum kernel $K_\infty(x_i,x_j)$ at the same rate (Corollary~\ref{cor:kernel_convergence}).
    \item[(ii)] \textbf{(Discretization mismatch bound.)} For any two nested discretizations $\mathcal{G}_n\subseteq\mathcal{G}_{n'}$ of the same manifold $M$ with $n \leq n'$ (e.g., obtained by mesh decimation), the attention kernel values at common nodes $x_i, x_j\in\mathcal{G}_n$ satisfy
    $$\left|\langle \tilde{\phi}_i^n, \tilde{\phi}_j^n \rangle - \langle \tilde{\phi}_i^{n'}, \tilde{\phi}_j^{n'} \rangle\right| = \mathcal{O}\!\left(n^{-1/(m+4)}\right) + \mathcal{O}\!\left(r^{-1/2}\right),$$
    where $\tilde\phi_i^n$ is the $i$-th row of $\widetilde\Phi^n = f(P_n)\cdot R^n$. The two sources combine additively via the triangle inequality: the first is the discretization mismatch error, the second the statistical estimation error from the random embedding.
    \item[(iii)] \textbf{(Gauge invariance.)} GIST's attention kernel $K(i,j)=[f(P)f(P)^\top]_{ij}$ is exactly invariant under all gauge transformations of the eigenvector basis (sign flips, rotations within degenerate eigenspaces, solver choices). Combined with (ii), this ensures that learned parameters transfer across discretizations (discretization-invariance): the attention weights converge to a gauge-invariant continuum limit that does not depend on the arbitrary eigenvector conventions of any particular discretization.
\end{enumerate}
\end{proposition}

The proof proceeds in three stages corresponding to the key technical components.

\subsubsection{Stage 1: spectral convergence and Green's function}

\begin{proposition}\label{prop:stage1_spectral_convergence}
	Let $\mathcal{L}_n$ be the (appropriately rescaled) graph Laplacian of a graph $\mathcal{G}_n$ obtained by sampling $n$ nodes uniformly from a compact $m$-dimensional Riemannian manifold $M$, with the bandwidth normalization required by \citet{calder2022improved} so that $\mathcal{L}_n$ converges spectrally to the (unbounded) Laplace-Beltrami operator $\Delta_M$. Let $\lambda_k^n, u_k^n$ be eigenvalues and eigenvectors of $\mathcal{L}_n$ in this normalization, and $\mu_k, \psi_k$ the eigenvalues and eigenfunctions of $\Delta_M$ on $M$. Then:
	\begin{enumerate}
		\item[(a)] Spectral convergence: $|\lambda_k^n - \mu_k| = \mathcal{O}(n^{-1/(m+4)})$ and $\|u_k^n - \psi_k\|_{L^2} = \mathcal{O}(n^{-1/(m+4)})$ (up to log factors).
		\item[(b)] Green's function convergence: the discrete Green's function on $\mathcal{G}_n$ (the inner product of Laplacian eigenmaps, see \eqref{eq:laplacian_eigenmaps}) converges to the continuum Green's function:
		$$\sum_{\lambda_k^n > 0} \frac{1}{\lambda_k^n}(u_k^n)_i(u_k^n)_j = G_M(x_i, x_j) + \mathcal{O}(n^{-1/(m+4)}),$$
		where $G_M(x_i, x_j) = \sum_{k=1}^{\infty} \frac{1}{\mu_k}\psi_k(x_i)\psi_k(x_j)$ is the Green's function of $\Delta_M$ on $M$.
	\end{enumerate}
\end{proposition}

\begin{proof}
	Part (a) follows from the spectral convergence theory for graph Laplacians on manifolds \citep{calder2022improved}: with appropriate graph construction, both eigenvalues and eigenvectors of $\mathcal{L}_n$ converge to those of $\Delta_M$ at rate $n^{-1/(m+4)}$ (up to log factors).
	Part (b) follows from part (a) by termwise convergence of the eigenfunction expansion of the pseudoinverse of $\mathcal{L}_n$.
\end{proof}

\begin{corollary}[Kernel convergence for general spectral filters]
\label{cor:kernel_convergence}
Let $\mu_1>0$ be the spectral gap of $\Delta_M$, and let $f$ be a spectral filter satisfying (a) $g(\lambda):=f(\lambda)^2$ is Lipschitz on $[\mu_1,\infty)$ with Lipschitz constant $L_g$, and (b) $\sum_l f(\mu_l)^2 < \infty$ (trace-class). Then the discrete kernel
$$K_n(i,j):=\sum_{\lambda_l^n>0} f(\lambda_l^n)^2\,(u_l^n)_i(u_l^n)_j$$
(projecting out the kernel of $\mathcal{L}_n$; equivalent to $[f(P_n)f(P_n)^\top]_{ij}$ when $f$ has no pole at $0$) converges to the continuum kernel $K_\infty(x_i,x_j):=\sum_l f(\mu_l)^2\,\psi_l(x_i)\psi_l(x_j)$ at rate $\mathcal{O}(n^{-1/(m+4)})$.
\end{corollary}

\begin{proof}
By Proposition~\ref{prop:stage1_spectral_convergence}(a), $|\lambda_k^n - \mu_k|=\mathcal{O}(n^{-1/(m+4)})$ and $\|u_k^n - \psi_k\|_{L^2}=\mathcal{O}(n^{-1/(m+4)})$ for each fixed $k$.
The Lipschitz condition on $g$ over $[\mu_1,\infty)$ gives $|g(\lambda_k^n)-g(\mu_k)|\leq L_g|\lambda_k^n - \mu_k|=\mathcal{O}(n^{-1/(m+4)})$ for $\lambda_k^n,\mu_k\geq\mu_1$.
Combined with the eigenvector convergence (each term in the sum carries eigenvalue and eigenvector errors at the same rate), this yields per-term convergence at rate $\mathcal{O}(n^{-1/(m+4)})$.
The trace-class assumption $\sum_l f(\mu_l)^2 < \infty$ provides the uniform tail control needed to upgrade per-term convergence to convergence of the full sum at the same rate.

The Laplacian-eigenmap case ($f(\lambda)=\lambda^{-1/2}$) falls outside the Lipschitz hypothesis because $g(\lambda)=1/\lambda$ has a pole at $\lambda=0$; that case is handled directly by Proposition~\ref{prop:stage1_spectral_convergence}(b), where the kernel of $\mathcal{L}_n$ is projected out and $K_\infty$ specializes to the Green's function $G_M$.
\end{proof}

\begin{remark}[Rescaled variants of $\mathcal{L}_n$]
\label{rem:rescaled_variants}
Corollary~\ref{cor:kernel_convergence} extends to any operator $P_n$ obtained from $\mathcal{L}_n$ by a smooth eigenvalue transformation: the spectrum of $P_n$ is then a smooth reparameterization of $\sigma(\mathcal{L}_n)$, and the eigenvalue and eigenvector convergence of Proposition~\ref{prop:stage1_spectral_convergence} transfers at the same rate $\mathcal{O}(n^{-1/(m+4)})$ up to bounded similarity factors. A smooth filter $f$ on $P_n$ is equivalently a smooth filter on $\mathcal{L}_n$, so the Lipschitz hypothesis of Corollary~\ref{cor:kernel_convergence} is preserved. Our FastRP implementation (Algorithm~\ref{code:fastrp}) is the case $P_n = D^{-1}A_n = \mathds{1} - \mathcal{L}_{\mathrm{rw},n}$ with $f(P)=\sum_{i=1}^{k} P^i$, both smooth.
\end{remark}

\subsubsection{Stage 2: spectral embeddings as kernel estimators}

We consider the general class of approximate spectral embeddings $\widetilde\Phi = f(P)\cdot R\in\mathbb{R}^{N\times r}$, where $P$ is a graph operator, $f$ is a spectral filter (see \S\ref{sec:approach} for definitions), and $R\in\mathbb{R}^{N\times r}$ is a random matrix with $\mathbb{E}[RR^\top]=\mathds{1}$.
We write $\phi_i = f(P)_{i\cdot}^\top\in\mathbb{R}^N$ for the $i$-th row of $f(P)$ (the ``exact'' spectral embedding of node $i$) and $\tilde\phi_i\in\mathbb{R}^r$ for the $i$-th row of $\widetilde\Phi$ (its random projection); Laplacian eigenmaps~\eqref{eq:laplacian_eigenmaps} are the special case $f(\lambda)=\lambda^{-1/2}$, $P=\mathcal{L}$. The node-indexed $\phi_i$ used here differs vectorwise from the eigenvalue-indexed coordinates $(\phi_i)_k=\lambda_k^{-1/2}(u_k)_i$ of \eqref{eq:laplacian_eigenmaps} (the two are related by an orthogonal change of basis), but they agree at the inner-product level, so the kernel statements below are independent of which representation is used.

\begin{lemma}[Kernel estimator]
\label{lem:kernel_estimator}
For any spectral filter $f$ and random matrix $R\in\mathbb{R}^{N\times r}$ with $\mathbb{E}[RR^\top]=\mathds{1}$, the inner products of $\widetilde\Phi=f(P)\cdot R$ satisfy
$$\mathbb{E}\!\left[\langle\tilde\phi_i,\tilde\phi_j\rangle\right] = K(i,j) := [f(P)f(P)^\top]_{ij}.$$
When $P$ is symmetric with eigenvectors $u_l$, this reduces to $K(i,j) = \sum_l f(\lambda_l)^2\,(u_l)_i(u_l)_j$.
\end{lemma}

\begin{proof}
$\mathbb{E}[\widetilde\Phi\,\widetilde\Phi^\top] = f(P)\,\mathbb{E}[RR^\top]\,f(P)^\top = f(P)\cdot \mathds{1}\cdot f(P)^\top = f(P)f(P)^\top$.
Taking the $(i,j)$ entry gives the result.
\end{proof}

\begin{proposition}[Exact gauge invariance of $K$]
\label{prop:exact_gauge_inv}
The kernel $K(i,j)=[f(P)f(P)^\top]_{ij}$ is exactly invariant under all gauge transformations of the eigenvector basis: sign flips $u_l\to s_l u_l$, orthogonal rotations within degenerate eigenspaces, and different numerical solver choices.
\end{proposition}

\begin{proof}
$K(i,j)=[f(P)f(P)^\top]_{ij}$ is an entry of the matrix $f(P)f(P)^\top$, which is a function of $P$ alone and does not depend on the choice of eigenvector basis. All gauge transformations leave $P$ unchanged, so $K$ is unchanged.
\end{proof}

\begin{lemma}[Concentration bound]
\label{lem:concentration}
Let $\mathbf{a}=f(P)_{i\cdot}^\top$ and $\mathbf{b}=f(P)_{j\cdot}^\top$ be the $i$-th and $j$-th rows of $f(P)$, and let $R\in\mathbb{R}^{N\times r}$ have i.i.d.\ sub-Gaussian entries with mean zero and variance $1/r$. Then:

\textbf{(i) Tail bound.} By the Hanson-Wright inequality \citep{rudelson2013hansonwright}, for any $t>0$:
$$\Pr\!\left[\left|\langle\tilde\phi_i,\tilde\phi_j\rangle - K(i,j)\right| > t\right]
\leq 2\exp\!\left(-c\cdot r\cdot\min\!\left(\frac{t^2}{\|\mathbf{a}\|^2\|\mathbf{b}\|^2},\frac{t}{\|\mathbf{a}\|\|\mathbf{b}\|}\right)\right),$$
where $c>0$ is an absolute constant depending only on the sub-Gaussian norm of the entries of $R$.

\textbf{(ii) Scaling of $r$.} Fix a target accuracy $\varepsilon>0$ and a failure probability $\delta\in(0,1)$. Setting $t=\varepsilon\|\mathbf{a}\|\|\mathbf{b}\|$ in the bound above, and noting that for $\varepsilon\leq 1$ the minimum is achieved by the quadratic term $t^2/(\|\mathbf{a}\|^2\|\mathbf{b}\|^2)=\varepsilon^2$, we require
$$2\exp(-c\cdot r\cdot \varepsilon^2) \leq \delta
\qquad\Longleftrightarrow\qquad
r \geq \frac{\log(2/\delta)}{c\,\varepsilon^2}.$$
Therefore $r = \mathcal{O}(\varepsilon^{-2}\log(1/\delta))$ random-projection dimensions suffice to guarantee
$$\left|\langle\tilde\phi_i,\tilde\phi_j\rangle - K(i,j)\right| \leq \varepsilon\,\|\mathbf{a}\|\|\mathbf{b}\|$$
with probability at least $1-\delta$. Equivalently, for fixed $\delta$, the estimation error is $\mathcal{O}(\|\mathbf{a}\|\|\mathbf{b}\|/\sqrt{r})$.
\end{lemma}

\begin{proof}
Write $R=[\mathbf{r}_1\;\cdots\;\mathbf{r}_r]$ where each column $\mathbf{r}_s\in\mathbb{R}^N$ has i.i.d.\ sub-Gaussian entries with mean zero and variance $1/r$. Then
$$\langle\tilde\phi_i,\tilde\phi_j\rangle
= \sum_{s=1}^{r}(\mathbf{a}^\top\mathbf{r}_s)(\mathbf{b}^\top\mathbf{r}_s)
= \mathbf{a}^\top\!\left(\sum_{s=1}^{r}\mathbf{r}_s\mathbf{r}_s^\top\right)\!\mathbf{b}
= \mathbf{a}^\top(RR^\top)\mathbf{b}.$$
Each summand $(\mathbf{a}^\top\mathbf{r}_s)(\mathbf{b}^\top\mathbf{r}_s)$ is independent with mean $\mathbb{E}[\mathbf{a}^\top\mathbf{r}_s\mathbf{r}_s^\top\mathbf{b}] = (1/r)\,\mathbf{a}^\top\mathbf{b} = (1/r)\,K(i,j)$, so $\mathbb{E}[\langle\tilde\phi_i,\tilde\phi_j\rangle]=K(i,j)$. The product $(\mathbf{a}^\top\mathbf{r}_s)(\mathbf{b}^\top\mathbf{r}_s)$ can be written as $\mathbf{r}_s^\top(\mathbf{a}\mathbf{b}^\top)\mathbf{r}_s$, a quadratic form in a sub-Gaussian random vector; symmetrizing the (asymmetric) form $\mathbf{a}\mathbf{b}^\top$ to $M:=\tfrac12(\mathbf{a}\mathbf{b}^\top+\mathbf{b}\mathbf{a}^\top)$ leaves the value of the quadratic form unchanged and gives matrix norms $\|M\|_F\leq\|\mathbf{a}\|\|\mathbf{b}\|$ and $\|M\|_{\mathrm{op}}\leq\|\mathbf{a}\|\|\mathbf{b}\|$. Summing $r$ such terms and applying the Hanson-Wright inequality \citep{rudelson2013hansonwright} to $\langle\tilde\phi_i,\tilde\phi_j\rangle - K(i,j) = \mathbf{a}^\top(RR^\top - \mathds{1})\mathbf{b}$ yields the tail bound in part (i). Part (ii) follows by inverting this bound.
\end{proof}

\subsubsection{Stage 3: gauge invariance of GIST}

GIST computes attention weights via linear attention with feature map $\varphi=\mathrm{ReLU}$ (as in Algorithm~\ref{code:gauge_invariant_attention}): the effective logits are $e_{ij}=\langle\varphi(\tilde\phi_i),\varphi(\tilde\phi_j)\rangle$. The gauge invariance of $e_{ij}$ rests on a sample-wise observation: for any fixed realization of the random matrix $R$, the embedding $\tilde\phi_i=[f(P)R]_i$ is a deterministic function of $P$ (computed without any eigendecomposition), so gauge transformations of the eigenvector basis, which leave $P$ unchanged, leave $\tilde\phi_i$, $\varphi(\tilde\phi_i)$, and $e_{ij}$ exactly unchanged. The argument is distribution-free and applies uniformly to any random-projection scheme of the form \eqref{eq:spectral_embedding_class}. In expectation, $\mathbb{E}[\langle\tilde\phi_i,\tilde\phi_j\rangle]=K(i,j)=[f(P)f(P)^\top]_{ij}$ (Lemma~\ref{lem:kernel_estimator}), tying the population-level attention to the gauge-invariant kernel $K$ (Proposition~\ref{prop:exact_gauge_inv}). This establishes part~(iii) of Proposition~\ref{prop:neural_operator_full}.

\subsubsection{Discretization mismatch error analysis}

Combining Stages~1 and~2, for two nested discretizations $\mathcal{G}_n\subseteq\mathcal{G}_{n'}$ of the same manifold $M$ and any pair of common nodes $i,j\in\mathcal{G}_n$, the attention kernel mismatch is bounded by the triangle inequality:
\begin{align*}
\left|\langle\tilde\phi_i^n,\tilde\phi_j^n\rangle - \langle\tilde\phi_i^{n'},\tilde\phi_j^{n'}\rangle\right|
&\leq \underbrace{\left|\langle\tilde\phi_i^n,\tilde\phi_j^n\rangle - K_n(i,j)\right|}_{\mathcal{O}(1/\sqrt{r})\text{ by Lemma~\ref{lem:concentration}}}
+ \underbrace{\left|K_n(i,j) - K_{n'}(i,j)\right|}_{\mathcal{O}(n^{-1/(m+4)})\text{ by Corollary~\ref{cor:kernel_convergence}}}
+ \underbrace{\left|K_{n'}(i,j) - \langle\tilde\phi_i^{n'},\tilde\phi_j^{n'}\rangle\right|}_{\mathcal{O}(1/\sqrt{r})\text{ by Lemma~\ref{lem:concentration}}},
\end{align*}
where $K_n(i,j)=[f(P_n)f(P_n)^\top]_{ij}$ denotes the kernel on discretization $\mathcal{G}_n$.
The middle term follows from Corollary~\ref{cor:kernel_convergence}: both $K_n$ and $K_{n'}$ converge to the same continuum kernel $K_\infty(x_i,x_j)=\sum_l f(\mu_l)^2\psi_l(x_i)\psi_l(x_j)$ at rate $\mathcal{O}(n^{-1/(m+4)})$, so by the triangle inequality $|K_n(i,j)-K_{n'}(i,j)|=\mathcal{O}(n^{-1/(m+4)})$.
The total bound is therefore
\begin{equation}\label{eq:kernel_mismatch}
\left|\langle\tilde\phi_i^n,\tilde\phi_j^n\rangle - \langle\tilde\phi_i^{n'},\tilde\phi_j^{n'}\rangle\right| = \mathcal{O}\!\left(\frac{1}{\sqrt{r}}\right) + \mathcal{O}\!\left(n^{-1/(m+4)}\right),
\end{equation}
with the two error sources combining additively via the triangle inequality. Choosing $r=\mathcal{O}(\varepsilon^{-2})$ and $n=\mathcal{O}(\varepsilon^{-(m+4)})$ achieves total error $\varepsilon$.

\subsubsection{From kernel mismatch to output mismatch}

The analysis above bounds the gauge-invariant attention kernel. The following corollary propagates that bound through the linear-attention normalization of a gauge-invariant spectral self-attention block, completing the chain from part~(ii) of Proposition~\ref{prop:neural_operator_full} to a discretization-mismatch bound on the corresponding output.

\begin{corollary}[Output mismatch for a gauge-invariant attention block]
\label{cor:output_mismatch}
Consider a gauge-invariant spectral self-attention block of GIST applied to node features $x_i$, with shared spectral embeddings $\tilde\phi_i$ at the input. Assume:
\begin{enumerate}
    \item[(a)] \emph{Bounded embeddings:} $\sup_{n,i}\|f(P_n)_{i\cdot}\|\leq\rho$ and $\|\tilde\phi_i^n\|\leq\rho$ (the former controls the $\|\mathbf{a}\|\|\mathbf{b}\|$ scaling in Lemma~\ref{lem:concentration}).
    \item[(b)] \emph{Bounded features:} $\|x_i\|,\|v_j\|\leq B$, and node features agree at common nodes across nested discretizations.
    \item[(c)] \emph{Bounded weights:} all learnable weight matrices have spectral norm at most $w$.
    \item[(d)] \emph{Non-degenerate normalizer:} $\inf_{n,i}\frac{1}{N}\sum_j \varphi(\tilde\phi_i^n)^\top\varphi(\tilde\phi_j^n) \geq Z_* > 0$ (the average attention weight is bounded away from zero). With $q_i=k_i=\tilde\phi_i$ this implies the linear-attention normalizer $Z_i^n = \sum_j \varphi(q_i)^\top\varphi(k_j) \geq N\,Z_*$, scaling linearly with the discretization.
    \item[(e)] \emph{Lipschitz feature map:} $\varphi$ is $\mathrm{Lip}_\varphi$-Lipschitz on the ball of radius $\rho$.
    \item[(f)] \emph{Lipschitz Gram regularity of the population kernel:} $\bar k(i,j):=\mathbb{E}_R[\langle\varphi(\tilde\phi_i),\varphi(\tilde\phi_j)\rangle]$ depends Lipschitz-continuously on the Gram quantities $(\|f(P)_{i\cdot}\|^2, \|f(P)_{j\cdot}\|^2, \langle f(P)_{i\cdot}, f(P)_{j\cdot}\rangle)$ on the bounded set defined by (a).
\end{enumerate}
Let $o_i^n$ denote the output at node $i$ on discretization $\mathcal{G}_n$. Then, for any two nested discretizations $\mathcal{G}_n\subseteq\mathcal{G}_{n'}$ of the same manifold $M$ as in Proposition~\ref{prop:neural_operator_full} and any common node $i\in\mathcal{G}_n$,
$$\bigl\|o_i^n - o_i^{n'}\bigr\| \;=\; C\cdot\!\left[\mathcal{O}\!\left(n^{-1/(m+4)}\right) + \mathcal{O}\!\left(r^{-1/2}\right)\right],$$
where the constant $C$ depends polynomially on $\rho, B, w, Z_*^{-1}$ and $\mathrm{Lip}_\varphi$ but not on $n$ or $r$.
\end{corollary}

\begin{proof}[Proof sketch]
\emph{(1) From kernel mismatch to raw-weight mismatch.} The raw linear-attention weights of the gauge-invariant block are $e_{ij}=\langle\varphi(\tilde\phi_i),\varphi(\tilde\phi_j)\rangle$ (the ``effective logits'' of Stage~3), with normalizer $Z_i = \sum_j e_{ij}$. The kernel-mismatch bound \eqref{eq:kernel_mismatch} (instantiated with $\|\mathbf{a}\|,\|\mathbf{b}\|\leq\rho$) controls $|\langle\tilde\phi_i^n,\tilde\phi_j^n\rangle - \langle\tilde\phi_i^{n'},\tilde\phi_j^{n'}\rangle|$ uniformly in $i, j$, and the diagonal case $i=j$ controls $|\|\tilde\phi_i^n\|^2 - \|\tilde\phi_i^{n'}\|^2|$. By assumption~(f), the population kernel $\bar k(i,j):=\mathbb{E}_R[e_{ij}]$ is Lipschitz in these Gram quantities, so $|\bar k_n(i,j) - \bar k_{n'}(i,j)|=\mathcal{O}(n^{-1/(m+4)})$. Combined with $\mathcal{O}(r^{-1/2})$ concentration of $e_{ij}$ around $\bar k(i,j)$ (Hanson-Wright on the bounded set), the triangle inequality yields $|e_{ij}^n - e_{ij}^{n'}|=\mathcal{O}(n^{-1/(m+4)})+\mathcal{O}(r^{-1/2})$.

\emph{(2) From raw-weight mismatch to output mismatch.} The values are $v_j = W_v x_j$, gauge- and discretization-independent at common nodes by~(b)~and~(c). The map $\{e_{ij}\}_j \mapsto o_i = \sum_j e_{ij} v_j / Z_i$ is Lipschitz under (b) and (d), with the constant absorbing the linear scaling $Z_i\geq N Z_*$ of the normalizer guaranteed by~(d) so that the per-pair perturbation in step~(1) does not accumulate with $N$; the output mismatch then inherits the rate of step~(1). Linear projections, Lipschitz nonlinearities, and residual paths downstream of attention preserve the rate under (c).
\end{proof}

Assumptions (a)--(c) are standard regularity conditions for trained transformer blocks; (d) is generic for $\mathrm{ReLU}$ on non-trivial embeddings; (f) is satisfied by Gaussian random projections, where $\bar k$ is the first-order arc-cosine kernel \citep{cho2009kernel}.

\begin{remark}[Refinements of the bound]
\label{rem:cor_refinements}
Three technical points sharpen the statement above without changing the leading-order rate:

\emph{(i) Excluded filters.} Assumption~(a) requires a uniformly bounded discrete kernel diagonal $K_n(i,i)\leq\rho^2$ and therefore excludes filters whose continuum kernel diverges on the diagonal, notably $f(\lambda)=\lambda^{-1/2}$ on manifolds with $m\geq 2$ (where $K_\infty(x_i,x_i)=G_M(x_i,x_i)$ has the standard logarithmic ($m=2$) or $|x_i-x_j|^{2-m}$ ($m\geq 3$) Green's-function singularity). Trace-class filters with bounded continuum diagonal, such as heat kernels $f(\lambda)=e^{-t\lambda}$ and bounded polynomial filters (including the FastRP filter $f(P)=\sum_{i=1}^{k} P^i$), satisfy~(a).

\emph{(ii) Uniform statistical error.} The per-pair concentration of Lemma~\ref{lem:concentration} extends to all $\binom{n}{2}$ node pairs via a union bound, replacing $\mathcal{O}(r^{-1/2})$ by $\mathcal{O}(\sqrt{\log n / r})$.

\emph{(iii) Spatial quadrature.} The discrete output is a Monte-Carlo approximation of the underlying continuum operator integral; for i.i.d.\ uniform sampling the resulting quadrature error is $\mathcal{O}(n^{-1/2})$, dominated by the spectral rate $\mathcal{O}(n^{-1/(m+4)})$ and absorbed into it.
\end{remark}

\section{Algorithms}\label{app:algorithms}

Note that in the pseudocode we use bold notation for matrices and vectors ($\mathbf{A}, \mathbf{\Phi}, \mathbf{Q}$) and follow the row-vector convention standard in machine learning: $\mathbf{\Phi}\in\mathbb{R}^{N\times r}$ has nodes as rows and embedding dimensions as columns.
In the main text, we use non-bold notation for compactness, with $\phi_i \in \mathbb{R}^r$ denoting the $i$-th node embedding (the $i$-th row of $\mathbf{\Phi}$, treated as a column vector when forming inner products) and upper-case characters denoting matrices.
The symbol $\mathbf{1}_N$ denotes the all-ones column vector in $\mathbb{R}^N$.

\begin{algorithm}
	\caption{Spectral embeddings via FastRP \citep{chen2019fast}}
	\begin{algorithmic}[1]
		\Require Graph adjacency matrix $\mathbf{A}\in\mathbb{R}^{N \times N}$, embedding dimensionality $r$, iteration power $k$
		\Ensure Matrix of $N$ node graph positional embeddings $\mathbf{\Phi} \in \mathbb{R}^{N\times r}$
		\State Produce very sparse random projection $\mathbf{R} \in \mathbb{R}^{N \times r}$ according to \citet{li2006very}
		\State $\mathbf{P} \leftarrow \mathbf{D}^{-1/2}\cdot\mathbf{A}\cdot\mathbf{D}^{-1/2}$ the normalized adjacency matrix, where $\mathbf{D}$ is the degree matrix
		\State $\mathbf{\Phi_1}\leftarrow \mathbf{P}\cdot\mathbf{R}$
		\For{$i = 2$ to $k$}
		\State $\mathbf{\Phi}_i \leftarrow \mathbf{P}\cdot\mathbf{\Phi}_{i-1}$
		\EndFor
		\State $\mathbf{\Phi} = \mathbf{\Phi}_1 + \mathbf{\Phi}_2 + \cdots + \mathbf{\Phi}_k$
		\State \Return $\mathbf{\Phi}$
	\end{algorithmic}
	\label{code:fastrp}
\end{algorithm}

Below we provide pseudo-code for the core computations of GIST, the \emph{Gauge-Invariant Spectral Self-Attention} block and the \emph{Gauge-Equivariant Spectral Self-Attention} block.
For illustration purposes, we compare the algorithms to a stripped down implementation of self-attention.
Modifications relative to vanilla self-attention are highlighted in \textcolor{red}{red} (additions) and \textcolor{gray}{\st{strikethrough}} (removals).

\begin{algorithm}
	\caption{\textcolor{red}{Gauge-Invariant Spectral} Self-Attention}
	\begin{algorithmic}[1]
		\Require Node feature tokens $\mathbf{X} \in \mathbb{R}^{N \times d}$ \textcolor{red}{, graph positional embeddings $\mathbf{\Phi}\in \mathbb{R}^{N \times r}$}
		\Ensure Output sequence $\mathbf{O} \in \mathbb{R}^{N \times d}$ to be applied to features $\mathbf{X}$
		
		\State // \textbf{Compute attention matrices}
		\State \textcolor{gray}{\st{$\mathbf{Q} \leftarrow \mathbf{X}\cdot \mathbf{W}_Q$ where $\mathbf{W}_Q \in \mathbb{R}^{d \times d}$}}
		~\quad\textcolor{red}{$\mathbf{Q} \leftarrow \mathbf{\Phi}$}
		\State \textcolor{gray}{\st{$\mathbf{K} \leftarrow\mathbf{X}\cdot \mathbf{W}_K$ where $\mathbf{W}_K \in \mathbb{R}^{d \times d}$}}
		\quad\textcolor{red}{$\mathbf{K} \leftarrow \mathbf{\Phi}$}
		\State $\mathbf{V} \leftarrow \mathbf{X}\cdot \mathbf{W}_V$ where $\mathbf{W}_V \in \mathbb{R}^{d \times d}$
		
    \State // \textbf{Compute linear attention with feature map $\varphi(x)=\mathrm{ReLU}(x-\theta)$}, with trainable scalar $\theta$
		\State $\tilde{\mathbf{Q}}, \tilde{\mathbf{K}} \leftarrow  \varphi(\mathbf{Q}),   \varphi(\mathbf{K})$
		\State $\mathbf{S} \leftarrow \tilde{\mathbf{K}}^T \mathbf{V}$ \Comment{Compute key-value matrix: $\mathbb{R}^{r \times d}$}
		\State $\mathbf{Z} \leftarrow 1 / (\tilde{\mathbf{Q}} (\tilde{\mathbf{K}}^T \mathbf{1}_N) + \epsilon)$ \Comment{Normalization factors: $\mathbb{R}^{N}$}
		\State $\mathbf{O} \leftarrow (\tilde{\mathbf{Q}} \mathbf{S}) \odot \mathbf{Z}$ \Comment{Normalized output: element-wise product}
		
		\State \Return $\mathbf{O}$
	\end{algorithmic}
	\label{code:gauge_invariant_attention}
\end{algorithm}

\begin{algorithm}
	\caption{\textcolor{red}{Gauge-Equivariant Spectral} Self-Attention}
	\begin{algorithmic}[1]
		\Require Node feature tokens $\mathbf{X} \in \mathbb{R}^{N \times d}$ \textcolor{red}{, graph positional embeddings $\mathbf{\Phi}\in \mathbb{R}^{N \times r}$}
		\Ensure Output sequence $\mathbf{O} \in \mathbb{R}^{N \times \textcolor{red}{r}}$ \textcolor{red}{to be applied to graph positional embeddings $\mathbf{\Phi}$}
		
		\State // \textbf{Compute attention matrices}
		\State $\mathbf{Q} \leftarrow \mathbf{X}\cdot \mathbf{W}_Q$ where $\mathbf{W}_Q \in \mathbb{R}^{d \times d}$
		\State $\mathbf{K} \leftarrow\mathbf{X}\cdot \mathbf{W}_K$ where $\mathbf{W}_K \in \mathbb{R}^{d \times d}$
		\State \textcolor{gray}{\st{$\mathbf{V} \leftarrow \mathbf{X}\cdot \mathbf{W}_V$ where $\mathbf{W}_V \in \mathbb{R}^{d \times d}$}}
		\quad\textcolor{red}{$\mathbf{V} \leftarrow \mathbf{\Phi}$}
		
		\State // \textbf{Compute linear attention \citep{katharopoulos2020transformers}}
		\State $\tilde{\mathbf{Q}}, \tilde{\mathbf{K}} \leftarrow  \varphi(\mathbf{Q}),   \varphi(\mathbf{K})$
		\State $\mathbf{S} \leftarrow \tilde{\mathbf{K}}^T \mathbf{V}$ \Comment{Compute key-value matrix: $\mathbb{R}^{d \times r}$}
		\State $\mathbf{Z} \leftarrow 1 / (\tilde{\mathbf{Q}} (\tilde{\mathbf{K}}^T \mathbf{1}_N) + \epsilon)$ \Comment{Normalization factors: $\mathbb{R}^{N}$}
		\State $\mathbf{O} \leftarrow (\tilde{\mathbf{Q}} \mathbf{S}) \odot \mathbf{Z}$ \Comment{Normalized output: element-wise product}
		
		\State \Return $\mathbf{O}$
	\end{algorithmic}
	\label{code:gauge_equivariant_attention}
\end{algorithm}

\paragraph{Note on the choice of feature map $\varphi$.}
While in Algorithm~\ref{code:gauge_equivariant_attention} we do not need to impose that restriction, in Algorithm~\ref{code:gauge_invariant_attention} we use the feature map $\varphi(x) = \text{ReLU}(x - \theta)$ with a trainable scalar $\theta$, chosen for its simplicity, non-negativity (yielding valid attention weights in the linear-attention formulation), and Lipschitz continuity. ReLU also has a well-known correspondence with the first-order arc-cosine kernel \citep{cho2009kernel}: for Gaussian random features $\mathbf{a}, \mathbf{b} \in \mathbb{R}^r$, the inner product $\langle \varphi(\mathbf{a}), \varphi(\mathbf{b}) \rangle$ converges (as $r \to \infty$) to $k_1(\mathbf{a},\mathbf{b}) = \frac{1}{\pi}\|\mathbf{a}\|\|\mathbf{b}\|[\sin\vartheta + (\pi-\vartheta)\cos\vartheta]$, where $\vartheta$ is the angle between $\mathbf{a}$ and $\mathbf{b}$, a function of the Gram quantities $(\|\mathbf{a}\|, \|\mathbf{b}\|, \mathbf{a}^\top \mathbf{b})$.

The gauge invariance of GIST's attention weights is established at the sample level (Stage~3 of the proof of Proposition~\ref{prop:neural_operator_full}): for any fixed realization of $R$, $\tilde\phi_i=[f(P)R]_i$ is a deterministic function of $P$, so any element-wise feature map yields attention weights that are exactly invariant under all gauge transformations of the eigenvector basis. The argument therefore applies to ReLU and to any other element-wise feature map.

The original linear attention work by \citet{katharopoulos2020transformers} used $\varphi(x) = \text{elu}(x) + 1$, which is also non-negative and works well in practice; our framework accommodates either choice. Investigating alternative feature maps (e.g., polynomial kernels, random Fourier features for RBF-like kernels) is left for future work.

\section{Additional empirical studies}\label{app:empirical}

\subsection{Training details}\label{sec:training_details}

GIST models are trained with the AdamW optimizer, weight decay between zero and $5\times10^{-4}$, and a base learning rate between $2\times10^{-5}$ and $2\times10^{-3}$, with linear warmup followed by linear decay. These choices are selected on the validation split as part of the per-benchmark hyperparameter optimization (Section~\ref{sec:hp_robustness}). Loss functions follow the task: cross-entropy for single-label node classification (Cora, CiteSeer, PubMed, Photo, ogbn-arxiv), binary cross-entropy with logits for the multi-label PPI benchmark, and mean squared error for surface-pressure regression on AirfRANS, ShapeNet-Car, DrivAerNet, and DrivAerNet++.

The GIST architectures used in our experiments stack two to three multi-scale blocks with hidden dimension $d \in \{128, 256, 512\}$; spectral-embedding hyperparameters $(r, k)$ are tuned within the robust ranges identified in Section~\ref{sec:hp_robustness}. Spectral embeddings are precomputed once per graph and cached, so their cost is amortized across training and is not included in the per-step training time. Single-graph datasets are trained full-graph where memory permits; PPI and mesh datasets are trained in mini-batches sized to the per-graph memory budget.

Table~\ref{tab:hparams} lists the per-benchmark configurations and parameter counts; the corresponding configuration files are included with the source code at \url{https://github.com/ibm/gist-neurips-2026}.

\begin{table}[ht]
	\caption{GIST configurations of the released code. $L$: number of multi-scale blocks; $d$: hidden dimension; $d_b$: width of the feature and local branches; $(r,k)$: FastRP embedding dimension and number of power iterations; $\theta_0$: initial value of the trainable threshold $\theta$; Drop.: dropout; LR: learning rate; WD: weight decay; Warm.: warmup epochs; Batch: graphs per batch. All models use 4 attention heads and MLP ratio 2. Baseline results are taken from the cited publications, with the configurations reported there.}
	\label{tab:hparams}
	\centering
	\footnotesize
	\setlength{\tabcolsep}{3pt}
	\begin{tabular}{lrcccccccccccc}
		\toprule
		Dataset & Params & $L$ & $d$ & $d_b$ & $r$ & $k$ & $\theta_0$ & Drop. & LR & WD & Epochs & Warm. & Batch \\
		\midrule
		Cora         & 6.55M  & 2 & 256 & 256 & 512 & 128 & 2.0 & 0.8 & 1e-3   & 5e-4 & 100  & 1  & 1 \\
		CiteSeer     & 17.13M & 2 & 512 & 256 & 512 & 128 & 2.0 & 0.8 & 3e-4   & 5e-4 & 100  & 1  & 1 \\
		PubMed       & 3.94M  & 2 & 256 & 128 & 512 & 256 & 1.0 & 0.8 & 1e-4   & 5e-4 & 1000 & 10 & 1 \\
		Photo        & 4.01M  & 2 & 256 & 128 & 512 & 256 & 2.0 & 0.7 & 1e-3   & 0                & 500  & 5  & 1 \\
		ogbn-arxiv   & 6.22M  & 2 & 256 & 256 & 512 & 256 & 1.0 & 0.7 & 2e-3   & 0                & 1000 & 10 & 1 \\
		PPI          & 5.83M  & 3 & 256 & 128 & 256 & 64  & 1.0 & 0.3 & 2.5e-4 & 0                & 1000 & 10 & 2 \\
		AirfRANS     & 6.18M  & 2 & 256 & 256 & 256 & 64  & 1.0 & 0.0 & 3e-5   & 0                & 300  & 0  & 1 \\
		ShapeNet-Car & 5.79M  & 3 & 256 & 128 & 256 & 64  & 1.0 & 0.0 & 1e-4   & 0                & 200  & 0  & 1 \\
		DrivAerNet   & 4.71M  & 3 & 256 & 64  & 256 & 64  & 1.5 & 0.0 & 2.5e-5 & 1e-5 & 100  & 2  & 1 \\
		\bottomrule
	\end{tabular}
\end{table}

\subsection{GIST hyperparameter robustness}\label{sec:hp_robustness}

In order to study the sensitivity of GIST's performance to variations in its spectral embedding hyperparameters, we train multiple simplified GIST architectures (two-block Gauge-Invariant Spectral Self-Attention linear transformers) on the Cora benchmark while varying the power iteration parameter $k$ and the embedding dimension $r$ of the spectral embedding.
These parameters directly control the quality of spectral embeddings while balancing computational efficiency.

As shown in Figure~\ref{fig:fastrp_robustness}, GIST exhibits robust performance across a wide range of both parameters.
The left panel sweeps $k$ with $r$ fixed, a clear but relatively shallow peak in accuracy around the optimal value of $k\approx 32$.
This on one hand suggests that even modest iteration counts are sufficient to capture the essential spectral structure, but also indicates that the choice of the power iteration is quite robust.
The right panel varies embedding dimension $r$ with $k$ fixed, demonstrating a smooth monotonic improvement as $r$ increases.
Crucially, saturation occurs relatively quickly: performance gains beyond $r=256$ are marginal, validating our choice of reasonable embedding dimensions that maintain computational efficiency.

These results empirically validate two important properties: (1) GIST does not require extensive hyperparameter tuning around these spectral parameters, suggesting stable generalization; and (2) the linear end-to-end complexity achieved with modest $k$ and $r$ values is both computationally practical and empirically effective.
Combined with the gauge-invariance guarantees that prevent dependence on arbitrary spectral choices, these hyperparameters provide a principled and empirically robust way to control the approximation quality of spectral embeddings without sacrificing scalability.

\begin{figure}[ht]
	\begin{center}
		\includegraphics[width=\linewidth]{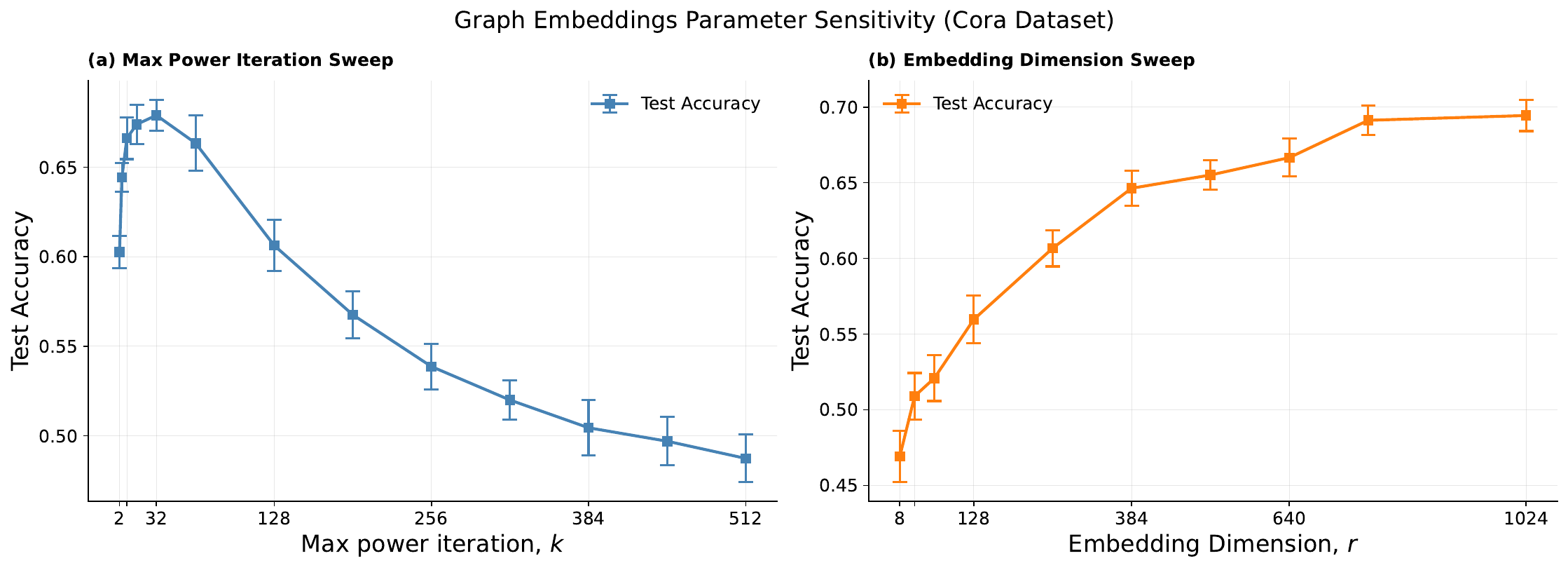}
	\end{center}
	\caption{
		Sensitivity study of GIST spectral embeddings parameters. The plots show the final test accuracy of a two-block Gauge-Invariant Spectral Self-Attention linear transformer trained on Cora while sweeping over the power iteration parameter $k$ with $r=256$ (left panel), and sweeping over the embedding dimension $r$ with $k=32$ (right panel). Test accuracy is fairly robust around the best value of either parameter. As expected, $r$ is monotonically related to higher performance, as higher $r$ correspond to better approximations of the eigenmaps. Accuracy conveniently saturates relatively fast, justifying the use of reasonably low $r$. The plots show mean test accuracy averaged across 10 seeds and corresponding standard deviation as error bars.}
	\label{fig:fastrp_robustness}
\end{figure}

\subsection{Convergence to the exact gauge-invariant kernel}\label{sec:kernel_convergence}

Figure~\ref{fig:kernel_convergence} checks the statistical term of Proposition~\ref{prop:neural_operator_informal} directly: on Cora, the arc-cosine attention weights estimated from finite-$r$ spectral embeddings converge to their exact, gauge-invariant limit at the predicted $\mathcal{O}(r^{-1/2})$ rate, for both Gaussian projections and the sparse projections used by FastRP.
The comparison isolates the kernel estimate.
Both curves run parallel to the $\mathcal{O}(r^{-1/2})$ reference line, with the sparse projection slightly above the Gaussian one and the gap narrowing as $r$ grows. At $r=256$, the typical embedding dimension in our experiments, the relative error is about $2$--$3\%$ for both projections.

\begin{figure}[ht]
	\centering
	\includegraphics[width=0.6\linewidth]{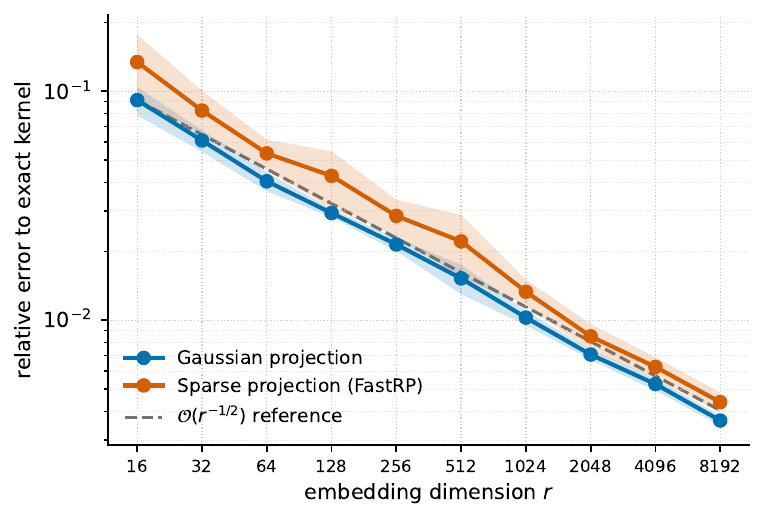}
	\caption{Convergence of the finite-$r$ kernel estimate to the exact gauge-invariant kernel on Cora. For each $r$, row-normalized first-order arc-cosine attention weights are computed from the embeddings $\widetilde\Phi=f(P)R$ (FastRP filter, $k=128$) and compared with the same weights computed from the exact rows of $f(P)$, the $r\to\infty$ limit; the y-axis is the relative Frobenius error. Both projections follow the $\mathcal{O}(r^{-1/2})$ rate of Lemma~\ref{lem:concentration}. Shaded bands: $\pm1$ std over 10 projection seeds.}
	\label{fig:kernel_convergence}
\end{figure}

\subsection{GIST scalability study}\label{sec:scalability}

\begin{figure}[ht]
	\begin{center}
		\includegraphics[width=\linewidth]{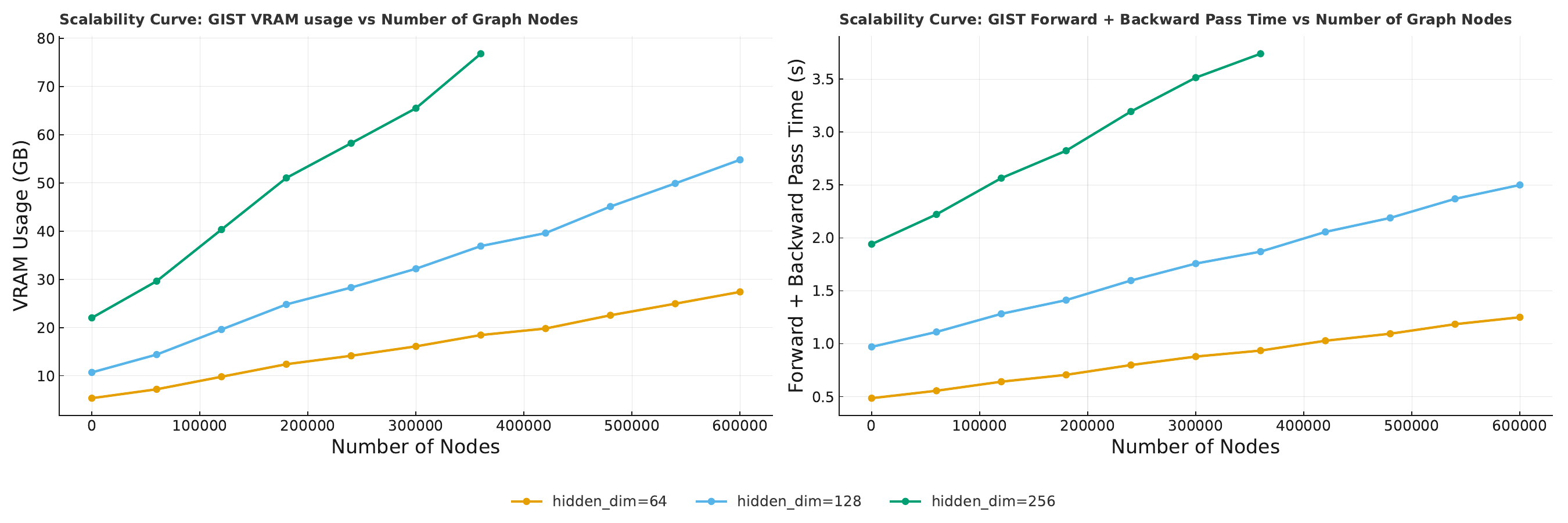}
	\end{center}
	\caption{
		Scalability study of GIST. All experiments use a fixed 3-layer model while varying the hidden dimensionality. VRAM consumption was measured as a function of the number of nodes in the input graph. Graph sizes were controlled using random node dropout applied to samples from the DrivAerNet dataset, enabling a systematic evaluation of memory scaling behavior.}
	\label{fig:gist_scalability}
\end{figure}

From a computational standpoint, the end-to-end cost of GIST is dominated by two components: spectral embedding generation and the subsequent transformer blocks. For the former, we employ a FastRP-style approximation in which the Laplacian spectral information is captured via repeated multiplication of a sparse random walk matrix with a low-dimensional random projection. Each power iteration requires $\mathcal{O}(|E| r)$ operations, where $|E|$ is the number of edges and $r$ is the embedding dimension, and the total cost over $k$ iterations is $\mathcal{O}(k |E| r)$. On meshes and graphs with bounded average degree, $|E| = \mathcal{O}(N)$, so the overall spectral embedding stage scales linearly in the number of nodes $N$. This embedding is computed once per graph and then reused across all GIST layers, so its cost is amortized over the full network depth.

The GIST layers themselves preserve this linear scaling. Each block combines: (i) a feature branch based on linear attention, (ii) a local branch using graph convolution followed by linear attention, and (iii) a global branch using Gauge-Invariant Spectral Self-Attention followed by linear attention, plus Gauge-Equivariant Spectral Self-Attention to refine the embeddings for the next block. In all cases, attention is implemented in the form $\varphi(Q)(\varphi(K)^\top V)$ with an element-wise feature map $\varphi(\cdot)$, which yields $\mathcal{O}(N d^2)$ complexity for $d$-dimensional features instead of the $\mathcal{O}(N^2 d)$ cost of quadratic attention. Together with the $\mathcal{O}(N)$ spectral embedding stage, this results in an overall complexity of $\mathcal{O}(N (d^2 + r d + r k))$ per forward pass, i.e., linear in the number of nodes. The empirical VRAM and wall-clock measurements in Figure~\ref{fig:gist_scalability} corroborate this analysis: both memory usage and forward time grow approximately linearly with the number of graph nodes for all hidden dimensions, up to graphs with hundreds of thousands of nodes sampled from DrivAerNet.

Table~\ref{tab:runtime} reports wall-clock times on DrivAerNet with the released code.

\begin{table}[ht]
	\caption{GIST wall-clock times on DrivAerNet (NVIDIA H100 80GB; per-mesh times are mean $\pm$ std over 50 test meshes, mean 409K nodes). FastRP embeddings are computed once per mesh and cached; inference runs in bf16 with \texttt{torch.compile}. Peak GPU memory is 32.9 GiB per GPU for training and 3.3 GiB for FastRP and inference together.}
	\label{tab:runtime}
	\centering
	\begin{tabular}{lcc}
		\toprule
		Stage & GPUs & Time \\
		\midrule
		Training (100 epochs) & 4 & 3.4 h \\
		FastRP embedding & 1 & $1.30 \pm 0.11$ s / mesh \\
		Inference (forward pass) & 1 & $0.049 \pm 0.003$ s / mesh \\
		\bottomrule
	\end{tabular}
\end{table}

\subsection{Multi-scale architecture ablation study}\label{sec:ablation}

To validate the design choices in the Multi-Scale GIST architecture, we systematically ablate each of the three parallel branches shown in Figure \ref{fig:diagram} (right panel) on the PPI dataset, where the full architecture achieves state-of-the-art performance (see Table \ref{tab:ppi}).
For each ablation, we train with identical hyperparameters but remove one branch: (1) feature processing, (2) local graph convolution, or (3) global spectral attention.
These ablation experiments are repeated across 10 random seeds.

Table \ref{tab:ablation} shows that all three branches contribute meaningfully, with performance drops ranging from 4.29\% to 7.90\%.
The local branch (Branch 2) has the strongest impact ($-7.90\%$), validating the EfficientViT-inspired design principle that local operations provide focused information complementing the diffuse patterns from linear attention.
The global spectral branch (Branch 3, $-4.29\%$) confirms that long-range dependencies are essential, while the feature branch (Branch 1, $-5.04\%$) provides complementary signal beyond structural information.
Overall, these results demonstrate that the Multi-Scale GIST effectively integrates complementary information sources.

\begin{table}[ht]
	\centering
	\caption{Ablation study on PPI dataset showing the contribution of each branch of the Multi-Scale GIST (see Figure \ref{fig:diagram}, right panel). Test accuracy is reported as percentage of baseline performance. Results averaged over 10 seeds with standard deviations indicated as uncertainty intervals.}
	\label{tab:ablation}
	\begin{tabular}{lcr}
		\toprule
		Ablation & Test Accuracy (\% baseline) & $\Delta$ (\%) \\
		\midrule
		No ablation & 100.0 & 0.0 \\
		Branch 1 (feature) & $95.0 \pm 2.8$ & $-5.0$ \\
		Branch 2 (local) & $92.1 \pm 1.8$ & $-7.9$ \\
		Branch 3 (global) & $95.7 \pm 3.7$ & $-4.3$ \\
		\bottomrule
	\end{tabular}
\end{table}

\makeatletter
\if@anonymous\else
\subsection{Pronunciation}

GIST is pronounced with a soft G, like \textit{gist} (the English word for the central point or essence of something), not with a hard G as in \textit{gift}.
Although insisting on a hard G won't necessarily be considered wrong.
After all, that is just a different gauge choice.
\fi
\makeatother

\newpage

\end{document}